\documentclass[sigconf]{acmart}

\pdfoutput=1

\usepackage{CJKutf8}
\usepackage{array}
\usepackage[utf8]{inputenc} % allow utf-8 input
\usepackage[T1]{fontenc}    % use 8-bit T1 fonts
\usepackage{colortbl}       % table row/column colors

\usepackage{hyperref}       % hyperlinks
\usepackage{url}            % simple URL typesetting
\usepackage{booktabs}       % professional-quality tables
\usepackage{amsfonts}       % blackboard math symbols
\usepackage{nicefrac}       % compact symbols for 1/2, etc.
\usepackage{microtype}      % microtypography
\usepackage{xcolor}         % colors
\usepackage{graphicx}
\usepackage{amsmath}
\usepackage{wrapfig}
\usepackage{color}
\usepackage{multirow}
\usepackage{comment}
\usepackage{caption}

\usepackage{subcaption}
\usepackage{pifont}

\newcommand{\ourmethod}{\textit{SGPO}}

\usepackage[normalem]{ulem}

\definecolor{myyellow}{RGB}{190,144,0}
\definecolor{mygreen}{RGB}{0,136,51}
\definecolor{myblue}{RGB}{0,102,204}

\newtheorem{thm}{\bf Theorem}[section]

\begin{document}

\begin{CJK*}{UTF8}{gbsn}


\title{Explore or Converge?
Stage-Guided Per-Step Optimization for Diffusion Models}

% Renye Yan and Jikang Cheng are co-first authors.

\author{Renye Yan}
\authornote{Renye Yan and Jikang Cheng contributed equally to this work.}
\affiliation{%
  \institution{Peking University}
  \city{Beijing}
  \country{China}
}

\author{Jikang Cheng}
\authornotemark[1]
\affiliation{%
  \institution{Peking University}
  \city{Beijing}
  \country{China}
}

\author{You Wu}
\affiliation{%
  \institution{Nanjing University}
  \city{Nanjing}
  \country{China}
}

\author{Wei Peng}
\affiliation{%
  \institution{Stanford University}
  \city{Stanford}
  \state{California}
  \country{USA}
}

\author{Zongwei Wang}
\affiliation{%
  \institution{Peking University}
  \city{Beijing}
  \country{China}
}

% Ling Liang and Yimao Cai are co-corresponding authors.

\author{Ling Liang}
\authornote{Ling Liang and Yimao Cai are the corresponding authors.}
\affiliation{%
  \institution{Peking University}
  \city{Beijing}
  \country{China}
}

\author{Yimao Cai}
\authornotemark[2]
\affiliation{%
  \institution{Peking University}
  \city{Beijing}
  \country{China}
}

% Short author list used in page headers.

\renewcommand{\shortauthors}{Yan et al.}

\begin{abstract}

Diffusion models have strong generative capabilities. However, their maximum likelihood training objective only focuses on reconstructing the data distribution, making it difficult to align with specific preferences. Reinforcement learning (RL) for preference alignment in diffusion models is promising but limited by reward sparsity. Since a single reward cannot support optimization, existing RL methods usually backpropagate the final reward to all previous steps. However, denoising is stage-wise, with distinct semantics and controllability. Repeating the final reward across all steps creates a temporal objective mismatch, encouraging reward shortcuts that lead to reward hacking. At the same time, due to reward backfilling, each time step receives the same reward, making it impossible to distinguish between actions, thereby weakening the optimization process.
To resolve this issue, we propose Stage-Guided Per-step Optimization (SGPO) for diffusion models, which jointly leverages signal-to-noise ratio and semantic changes to identify generation stages and adaptively assign stage-specific objectives. Early denoising is chaotic and far from the final reward, resulting in weak reward--behavior correlation. This stage should prioritize exiting the chaotic state. In the mid stage, the latent transitions to a stable structure, where the final reward better corresponds to generative behavior. Therefore, this stage optimizes the final reward while exploring diversity to avoid early convergence to a single mode.
In the late stage, the latent's core structure is largely fixed, and preference optimization mainly amplifies local details, risking overfitting. Therefore, stable convergence is preferred to avoid quality degradation. Results from \textbf{16} comparative experiments validate SGPO. Our method achieves \textbf{26.7\%} average gains in generative quality and \textbf{36.7\%} higher convergence speed.

% \textcolor{red}{现在这个名字太土了给我叫SwitchTune,另外可以叫分层强化学习但是考虑一下要引入不这些mm审稿人不一定懂}

\end{abstract}




%% information and the body of the document, and typically spans the page.





\maketitle

\section{Introduction}
\label{sec:intro}

\begin{figure}[t]
    \centering
    \includegraphics[width=1\linewidth]{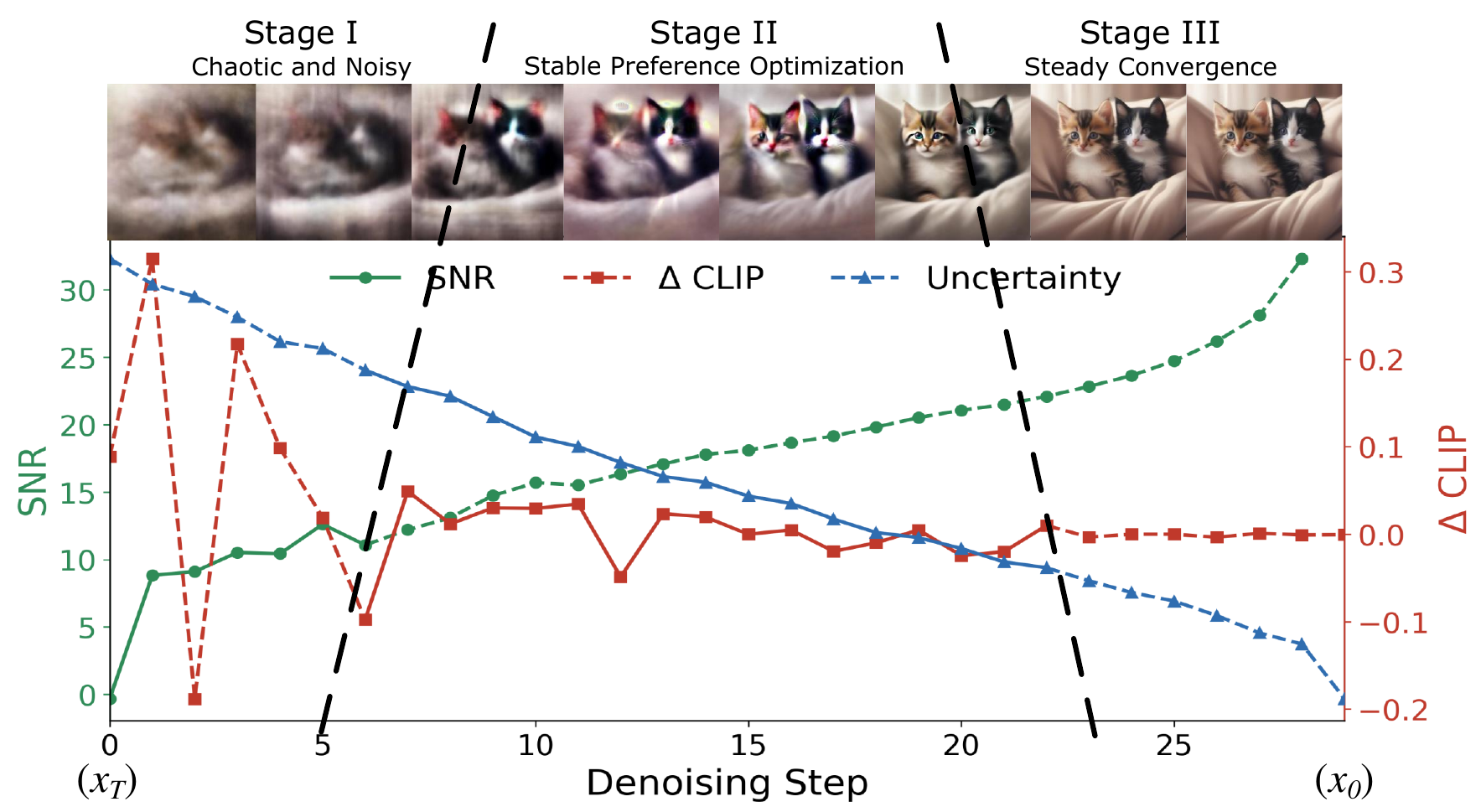}
    \caption{\textbf{Motivation}. We visualize adaptive semantic changes ($\Delta$ CLIP) between adjacent latents and step-wise signal-to-noise ratio (SNR) during denoising, revealing three adaptive stages: chaotic, structure-stable, and structure-convergent.}
    \label{fig:motivation}
\end{figure}

\begin{figure*}[t]
    \centering
    \includegraphics[width=1.0\linewidth]{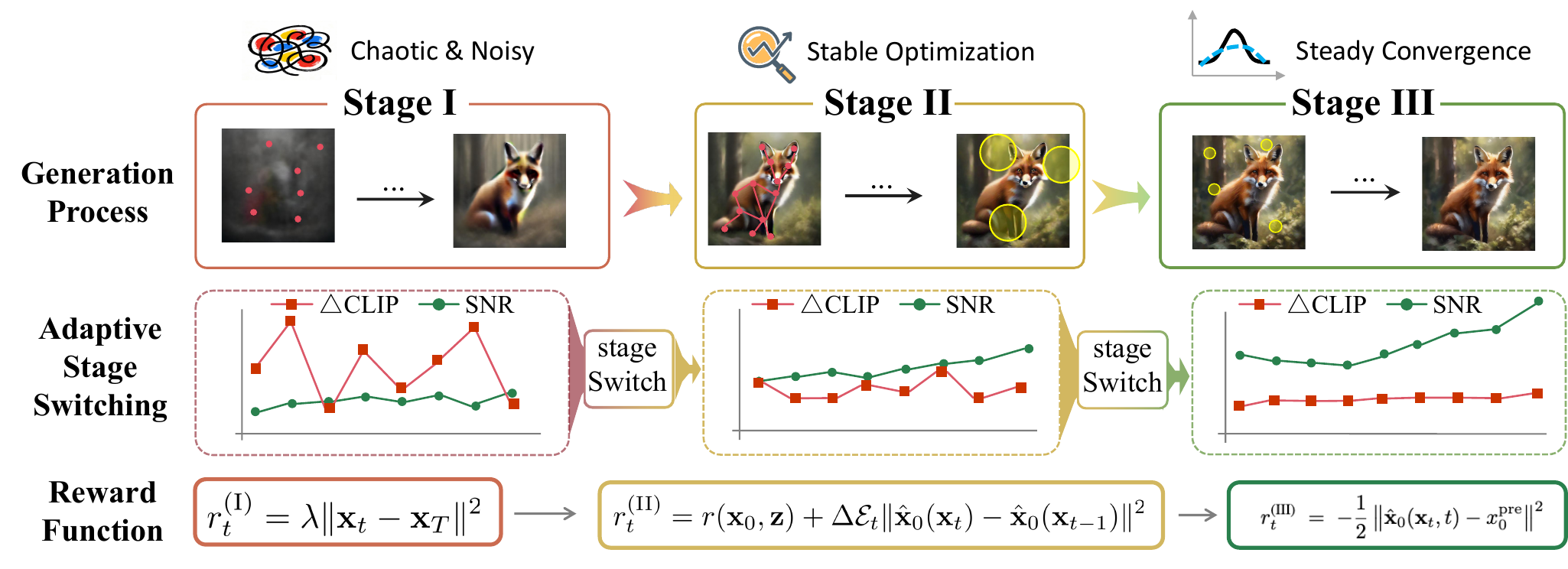}
    \caption{\textbf{Framework of \ourmethod.} As illustrated, RL training is adaptively stage-scheduled via continuous sensing of stage-wise semantic evolution to align with diffusion generation requirements.}
    \label{fig：framework}
    \vspace{-0.3cm}
    % 50k
\end{figure*}

Diffusion models~\citep{zhang2025trustclip,yan2026pixel,yan2026less} have recently achieved significant success in text-to-image generation~\citep{gandikota2025sliderspace} with diverse applications~\citep{zhang2025generalization}. 
However, the training objective of standard diffusion models primarily focuses on fitting real-data distributions, making it challenging to align directly with specific preferences. Since reinforcement learning (RL)~\citep{yan2024exploration,gan2024reflective,gan2024transductive} can directly optimize objective functions, it is regarded as a practical approach for aligning with preferences~\citep{black2023training,yan2025entropy}. However, the preference reward (final reward) can be accurately assessed only upon completion. The intermediate latent states lack direct feedback, resulting in a significant sparsity of rewards during RL training. This issue makes it difficult to support effective optimization using only the final reward (See Appendix Fig.~\ref{fig:rewardback}). Therefore, existing methods typically adopt a reasonable compromise strategy by propagating the fixed final reward backward to all denoising steps, thereby providing each step with the same reward signal~\citep{black2023training,fan2024reinforcement, yang2024using,yan2025entropy}.

Under sparse rewards, backpropagating the final reward is widely used; yet, it assumes that intermediate states across denoising stages are optimally equivalent. However, this assumption is not fully consistent with diffusion generative dynamics, which recent studies show to be stage-wise~\citep{choi2022perception,li2023autodiffusion,yi2024towards,xie2025dymo}. Different denoising stages fundamentally differ in their semantic structure, controllability, and their relationship to final output quality~\citep{yi2024towards,xie2025dymo}. In the early stage of generation, latents are primarily noise-dominated and only weakly correlated with the final reward. During the middle stage, the denoising process gradually reveals stable semantic structures. In the late stage, the overall semantics and structure of the latent are mostly finalized. Thus, indiscriminate final-reward application causes temporal objective mismatch, which encourages reward shortcuts and ultimately leads to reward hacking ~\citep{yan2025entropy,xie2025dymo}.

Mitigating reward hacking requires detecting stage transitions and adaptively aligning optimization with stage-specific generative dynamics and requirements; accordingly, we propose \ourmethod{}, an adaptive stage-aware training paradigm that aligns optimization focus with different phases of diffusion generation. Specifically, our method detects stage transitions by analyzing semantic changes and the signal-to-noise ratio (SNR), thereby dividing diffusion generation into three stages (Fig.~\ref{fig:motivation}). In the early stages, sharp fluctuations between adjacent latents indicate a high-noise chaotic regime, in which action and task reward attribution is unreliable. Thus, \ourmethod{} prioritizes exiting chaos over preference optimization. As semantic changes stabilize, denoising enters the middle stage, where latents approach the generation endpoint and  task reward and action correlation strengthens. Thus, this stage optimizes task (preference) rewards while using exploration to maintain diversity. In the late stage, global semantics and structure are largely fixed, while local details remain sensitive. Further preference optimization no longer alters the main structure but instead amplifies details to fit rewards. This behavior increases the risk of overfitting and destabilizing distributional convergence~\citep{yan2025entropy}. Thus, suppressing unnecessary drift and encouraging stable convergence is more appropriate.

% 共计 xxx 项实验结果表明，通过自适应地区分生成过程中探索与偏好奖励优化阶段，以及收敛引导阶段，所提出的方法在多个关键指标上取得了显著改进。具体而言，在动作与偏好奖励关联性最强、奖励提升最具收益的生成中期阶段，我们对偏好奖励进行了针对性的高质量优化，并引入带方向的探索机制，从而在偏好奖励指标上实现了约 xxx\% 的平均提升，同时显著缓解了奖励劫持问题，并带来了约 xxx\% 的多样性提升。此外，在生成后期通过显式引导生成轨迹向真实数据分布稳定收敛，模型获得了更稳健的生成行为，使其在不同生成任务中能够更可靠地控制整体结构与局部细节。由于避免了在偏好奖励收益有限的末尾阶段进行过度优化，引发yyyyy所提出的方法进一步提升了细节保真度（约提升 xxx）。\textcolor{red}{消融实验请康康给来一个全程进行奖励优化的ddpo，在中间段进行奖励优化的我们的，在末尾和中段都进行的。或者单独把美学的指标代尔塔变化可视化一下换个形式别跟现在多目标那个一样}

% The proposed method achieves significant improvements across multiple key metrics, as evidenced by xxx experimental results. These gains stem from its ability to adaptively distinguish between the optimization stages of exploration, preference reward, and convergence guidance. The mid-generation stage presents the most favorable conditions for optimization, as the correlation between actions and preference rewards is strongest, and enhancing rewards yields the greatest benefits. Leveraging this, we perform targeted, high-quality optimization of the preference reward and introduce a directional exploration mechanism. This leads to an average improvement of approximately xxx\% in preference reward metrics, while significantly alleviating the reward hacking problem and delivering about a xxx\% increase in diversity. In the late-generation stage, the model's trajectories are explicitly guided toward stable convergence with the real data distribution. This targeted stabilization allows more robust generative behavior, enabling reliable control over both global structure and local details across diverse tasks.
In summary, this paper makes the following contributions:
\begin{itemize}
    \item We propose a stage-aware RL framework for diffusion models that departs from uniform optimization by explicitly modeling stage-wise generative dynamics.

    \item We design stage-specific optimization objectives that align training signals with phase-dependent generation characteristics, effectively mitigating reward hacking while improving generation quality and diversity.

    \item Experimental results show that SGPO adaptively identifies diffusion stages and applies stage-specific RL strategies, achieving substantial gains across key metrics. By skipping the chaotic stage and conducting preference optimization and diversity exploration during the intermediate stage, SGPO improves reward scores by approximately 26.7\% and diversity by 8.51\%. This stage-aware optimization enables the training signal to better match the requirements of different generation phases. Stable late-stage convergence further enhances structural and detail controllability, mitigates over-optimization, and improves detail fidelity by 41.93\%.

\end{itemize}

% \begin{figure*}[t]
%     \centering
%     \begin{minipage}[b]{0.3\linewidth}
%         \centering
%         \includegraphics[width=1\linewidth]{sec/figure/reward_backfilling_.pdf}
%         \caption{The comparing results of Backpropagating final reward to 1) \textbf{Last}: the last step. 2) \textbf{Full}: full trajectory (widely adopted by existing methods). 3) \textbf{Part}: only the stable middle stage.}
%         \label{fig:rewardback}
%     \end{minipage}%
%     \hspace{0.04\linewidth} % 小的水平间距
%     \begin{minipage}[b]{0.6\linewidth}
%         \centering
%         \includegraphics[width=1\linewidth]{Image/diversity_ijcai.pdf}
%         \caption{\textbf{Generation diversity comparison}. SGPO yields higher diversity in shape, orientation, and color than baselines under equal training epochs.
%         % \textcolor{red}{3.5和数据集}
%         }
%         \label{fig:diverstiy}
%     \end{minipage}
% \end{figure*}

\section{Related Work}
\subsection{Reinforcement Learning Fine-tuning for Diffusion Models}
Recent advances have explored reinforcement learning (RL)~\citep{kaelbling1996reinforcement,chaudhari2025rlhf,zhang2025survey,tang2025deep,pippas2025evolution,balhara2025survey,michailidis2025reinforcement} as an effective paradigm for fine-tuning diffusion models toward specific objectives, such as aesthetic quality, human preference alignment, or task-specific constraints. Unlike standard diffusion training, which optimizes likelihood-based objectives, RL-based fine-tuning enables direct optimization of reward functions~\citep{lamba2025alignment,wagenmaker2025steering,zhao2022alphaholdem,gan2024reflective,schulman2017proximal,sampa2026reinforcement}.

Most existing approaches formulate diffusion sampling as a sequential decision-making process, where each denoising step corresponds to an action, and a preference model or evaluator assigns a reward to the final generated sample. To address the extreme reward sparsity inherent in this setup, these methods typically propagate the final-step reward backward to all denoising steps and apply policy gradient updates across the entire trajectory. This strategy has been shown to improve alignment performance in practice.

However, such approaches implicitly treat all denoising steps as equally relevant for reward optimization, overlooking the intrinsic stage-wise structure of the diffusion generation process. As discussed in recent studies~\citep{choi2022perception,li2023autodiffusion,yi2024towards,xie2025dymo}, this simplification can lead to optimization inefficiencies, reward hacking, and degradation of sample diversity, motivating the need for more structured and dynamics-aware reinforcement learning formulations.

\subsection{Reward Hacking}
\label{reward_hacking_main}
Reward hacking~\citep{eisenstein2023helping,skalse2022defining,pan2024feedback,laidlaw2024correlated,hadfield2017inverse} is a well-known failure mode in reinforcement learning, where an agent maximizes the reward function through unintended shortcuts rather than achieving the intended task objective. This phenomenon typically arises when the reward function is misaligned with the true optimization goal, leading the policy to exploit spurious correlations rather than learn meaningful behaviors. In Fig.~\ref{fig:hack_impre}, we show the typical hacking impressions on a simple animal prompt set after reward optimization.

In the context of diffusion model fine-tuning with reinforcement learning, reward hacking becomes particularly pronounced due to the sparse and delayed nature of preference rewards~\citep{fan2024reinforcement}. Most existing approaches compute a single reward only at the final generation step and propagate it backward to all denoising steps. This reward-backpropagation strategy implicitly assumes that all intermediate denoising states are equally informative and can be guided by a single reward signal. However, such an assumption ignores the intrinsic stage-wise dynamics of diffusion generation.

As a result, repeatedly optimizing an identical reward across consecutive denoising steps can encourage the model to discover shortcuts that improve the accumulated reward over time without genuinely improving generation quality. These shortcuts often manifest as mode collapse, repetitive structures, or over-amplified local patterns, leading to reward hacking and degradation of generative diversity. Reward hacking in diffusion model RL fine-tuning is not merely caused by reward sparsity, but fundamentally rooted in the temporal homogenization of rewards induced by indiscriminate reward backpropagation~\citep{zhai2025mira,clark2023directly,rafailov2024scaling,bai2025dragon}.

\section{Method}
\label{sec:method}
% 首先，在Sec.~\ref{per}中给出本文所使用的扩散模型强化学习框架下所包含的数学定义。随后，在Sec.~\ref{switching}中我们给出了自适应生成阶段感知机制的构建。最后，在Sec.~\ref{Optimization}中我们给出了生成过程中各阶段所需的优化目标。

In Sec.~\ref{per}, we provide the mathematical definitions of the diffusion models used in this paper within the RL framework. Subsequently, Sec.~\ref{switching} presents the construction of the adaptive generation stage perception mechanism. Finally, Sec.~\ref{Optimization} details the optimization objectives required for each stage during the generation process.

\subsection{Reinforcement Learning  for Diffusion Models}
\label{per}
\paragraph{Diffusion Models.}
A diffusion model defines a data distribution by reversing a predefined noising process.  
Starting from a clean sample $\mathbf{x}_0$, the forward 
diffusion process 
$q(\mathbf{x}_t \mid \mathbf{x}_{t-1})$ 
incrementally injects Gaussian noise, yielding a latent trajectory
$\mathbf{x}_0 \!\rightarrow\! \cdots \!\rightarrow\! \mathbf{x}_T$, 
pure noise $\mathbf{x}_T$ is distributed as 
$p(\mathbf{x}_T)=\mathcal{N}(0,I)$.  
The generative model is trained to learn the corresponding reverse-time transitions: $
p_\theta(\mathbf{x}_{0:T})
= p(\mathbf{x}_T)\prod_{t=1}^T 
p_\theta(\mathbf{x}_{t-1}\mid \mathbf{x}_t),
$
The parameterized reverse conditionals are learned to approximate the exact reverse posterior
$q(\mathbf{x}_{t-1}\mid \mathbf{x}_t,\mathbf{x}_0)$.  
Sampling begins from noise $\mathbf{x}_T$ and successively composes the learned reverse transitions to recover a data sample.
Let 
$
q\left(\mathbf{x}_t \mid \mathbf{x}_{t-1}\right)
:= \mathcal{N}\!\left(\mathbf{x}_t; \sqrt{1 - \beta_t}\,\mathbf{x}_{t-1},\, \beta_t \mathbf{I}\right),
$
where$\alpha_t := 1 - \beta_t,\qquad \bar{\alpha}_t := \prod_{k=1}^t \alpha_k,\quad (\bar{\alpha}_0 := 1)$.
Let $\epsilon_\theta(\mathbf{x}_t, t)$ denote the learned noise-prediction function.  
The induced estimate of the clean sample $\hat{\mathbf{x}}_0$ can be obtained as：
\begin{equation}
\label{def: pseudo x0}
\hat{\mathbf{x}}_0(\mathbf{x}_t)
= \frac{1}{\sqrt{\bar{\alpha}_t}}
\left(
    \mathbf{x}_t
    - \sqrt{1 - \bar{\alpha}_t}\,
      \epsilon_\theta(\mathbf{x}_t)
\right).
\end{equation}

\paragraph{Text-to-Image Tasks.}
For text-conditioned generation, a conditioning representation
$\mathbf{z}$, typically extracted from a natural-language prompt, is 
introduced to guide the reverse diffusion process.
The model therefore parameterizes
$
p_\theta(\mathbf{x}_{0:T}\mid \mathbf{z})
= p(\mathbf{x}_T)\prod_{t=1}^T 
p_\theta(\mathbf{x}_{t-1}\mid \mathbf{x}_t,\mathbf{z}),
$
so that each denoising step is guided by $\mathbf{z}$.  
This conditioning mechanism steers the reverse trajectory toward samples 
that are semantically consistent with the input text.

\paragraph{MDP Formulation.}

% 扩散模型的去噪轨迹可表示为$\tau_{\text{Diffusion}} = (\mathbf{x}_T, \mathbf{x}_{T-1}, \ldots, \mathbf{x}_0), \quad \tau_{\text{Diffusion}} \sim p_\theta(\mathbf{x}_{0:T})$,
% 其中 $x_T$ 为纯噪声起点，$x_0$ 为最终输出样本。我们对去噪过程进行Markov Decision Process（MDP）建模。其中 $p_\theta(\mathbf{x}_{0:T}\mid\mathbf{z})$ 定义为开展文生图任务的扩散模型；定义prompt为$\mathbf{z}$，其服从分布 $p(\mathbf{z})$。MDP的轨迹表示为$
% \tau_{\mathrm{MDP}} = (\mathbf{s}_0, \mathbf{a}_0, r_0, \mathbf{s}_1, \mathbf{a}_1, r_1, \ldots, \mathbf{s}_T, r_T), 
% \quad \tau_{\mathrm{MDP}} \sim p_\theta(\mathbf{s}_{0:T}, \mathbf{a}_{0:T-1})$.
% We formulate the denoising process as a Markov Decision Process (MDP)~\cite{puterman1990markov,gan2024reflective}.

We formulate the denoising process as a Markov Decision Process (MDP). 
% The prompt $\mathbf{z}$ following the distribution $p(\mathbf{z})$.
% The MDP trajectory is defined as $
% \tau_{\mathrm{MDP}} = (\mathbf{s}_0, \mathbf{a}_0, r_0, \mathbf{s}_1, \mathbf{a}_1, r_1, \ldots, \mathbf{s}_T, r_T)$.
In the MDP formulation, \( \mathbf{s}_t \) represents all image states (latents) during the diffusion model's denoising process. Each state \( \mathbf{s}_t \) corresponds to the intermediate state \( \mathbf{x}_{T-t} \) in the denoising sequence. The action \( \mathbf{a}_t \) includes both the noise prediction result and the newly sampled noise at the timestep $t$, i.e., \( \mathbf{a}_t = \mathbf{x}_{T-t-1} \). It determines the transition from the current state \( \mathbf{x}_{T-t} \) to the next state \( \mathbf{x}_{T-t-1} \). $\pi_\theta$ denotes the parameterized policy. $P_0$ is the initial state distribution. $\delta_{(\mathbf{z}, \mathbf{x_{T-t-1}})}$ represents the Dirac distribution centered at $(\mathbf{z}, \mathbf{x_{T-t-1}})$.
\( P(\mathbf{s}_{t+1} \mid \mathbf{s}_t, \mathbf{a}_t) \) denotes the \textit{state transition distribution}, which defines the conditional probability distribution of the next state \( \mathbf{s}_{t+1} \) given the current state \( \mathbf{s}_t \) and the action \( \mathbf{a}_t \). The detailed definitions of  the above components in MDP are as follows:
$\mathbf{s}_t = \left(\mathbf{z}, \mathbf{x}_{T-t}\right)$, $\mathbf{a}_t = \mathbf{x}_{T-t-1}$, $\pi_\theta\left(\mathbf{a}_t \mid \mathbf{s}_t\right) = p_\theta\left(\mathbf{x}_{T-t-1} \mid \mathbf{x}_{T-t}, \mathbf{z}\right)$, $P_0\left(\mathbf{s}_0\right) = \left(p(\mathbf{z}), \mathcal{N}(\mathbf{0}, \mathbf{I})\right)$, $P\left(\mathbf{s}_{t+1} \mid \mathbf{s}_t, \mathbf{a}_t\right) = \delta_{(\mathbf{z}, \mathbf{x_{T-t-1}})}$. The preference rewards of existing RL methods is defined as:
\begin{equation}
\label{fake_reward}
        R\left(\mathbf{s}_t, \mathbf{a}_t\right) \triangleq 
        \begin{cases}
            r(\mathbf{s}_{t+1}) = r\left(\mathbf{x}_0, \mathbf{z}\right) & \text{if } t=T-1, \\
            0 & \text{otherwise}.
        \end{cases}.
\end{equation}
% 在去噪过程中，除最后一步外的所有中间时间步的奖励均为零；仅当生成过程完成即去噪到达最后一个时间步 \( t = T - 1 \) 时，奖励模型才根据最终生成图像的质量给出精确的偏好评估得分。
During the denoising, the rewards for all intermediate timesteps are zero. The reward model provides an accurate preference score based on the quality of the generated image only when the generation process is complete, i.e., when denoising reaches the final timestep \( t = T - 1 \). During RL-based diffusion model fine-tuning, our objective is to maximize the expected reward of generated images under the prompt distribution $p(\mathbf{z})$ as follow: $
\max_{\theta} \ \mathbb{E}_{p(\mathbf{z})} \mathbb{E}_{p_\theta(\mathbf{x}_0 \mid \mathbf{z})}[r(\mathbf{x}_0, \mathbf{z})]$. Hereafter, \textbf{\textit{the preference reward (final reward)}} is denoted as $r\left(\mathbf{x}_0, \mathbf{z}\right)$, distinct from \textbf{\textit{the stage rewards}}.

\subsection{Justification of Stage-wise Training Priorities}

\begin{thm}\label{thm1}
	If the diffusion model is perfectly trained and the reverse generation follows the score-based SDE,  $\hat{\mathbf{x}}^t_0$ denotes the posterior distribution of $\mathbf{x}_0$ at time $t$, then $\operatorname{Var}(\hat{\mathbf{x}}^t_0 \mid \mathbf{x}_t)$ is a decreasing function during the generation process.
\end{thm}
The proofs of Theorem \ref{thm1} is shown in Appendix.
\noindent According to Theorem~\ref{thm1}, the posterior variance $\operatorname{Var}(\hat{\mathbf{x}}^t_0 \mid \mathbf{x}_t)$ decreases monotonically during the generation process. A decreasing variance implies that the model's uncertainty about $\mathbf{x}_0$ gradually diminishes over the generation process. This aligns with intuition: at the early stage of generation, the sample is almost pure noise, making it difficult for the model to infer the original data, so the uncertainty is high; as the time step decreases, more signal is recovered, and the model's generation becomes increasingly certain.

\noindent \textbf{Linking Theory and RL:} Reduced uncertainty induces progressively stabilized state transitions. High uncertainty yields chaotic latents, weak action–reward correlation, and high training variance (intuitively, the image is not yet formed, so finetuning the model in this phase can lead to inconsistent final results). In contrast, overly low uncertainty implies near-converged state transitions, saturated gains, and susceptibility to overfitting (intuitively, the image is mostly fixed, and further adjustments yield neglectable improvement). Based on this theory and Fig.~\ref{fig:motivation}, SGPO restricts the task (preference) reward to an intermediate-uncertainty regime (Stage II) with clear action-task (preference) reward attribution and evident gains. In the high-uncertainty region (Stage I), action-task (preference) reward attribution is extremely low, so this region is not suitable for optimizing the preference reward; instead, it should encourage quick denoising. Finally, when the uncertainty enters the extremely low region (Stage III), action-preference reward attribution basically converges, and there is no further space for optimizing the preference; forcing optimization at this stage easily leads to overfitting. Therefore, in this stage, we encourage stable convergence to the true distribution.
\subsection{Adaptive Switching Mechanism}
\label{switching}
In this subsection, we detect stage-wise generative dynamics of diffusion models via real-time semantic structure changes.
\paragraph{Semantic Structure of Generation.}
% 首先对于潜变量 $\mathbf{x}_t$ 使用Eq.~\ref{def: pseudo x0}预测对应的干净样本估计 $\hat{\mathbf{x}}_0(\mathbf{x}_t, t)$。

% 为了刻画生成过程中语义结构的演化特性，我们引入一个阶段敏感的语义信号 $\mathcal{E}_t$，其被定义为预测干净样本 $\hat{\mathbf{x}}_0(\mathbf{x}_t, t)$ 与条件输入（如文本提示）之间的语义对齐得分：
% Firstly, for the latent $\mathbf{x}_t$, we use Eq.~\ref{def: pseudo x0} to estimate the corresponding clean sample. 
%  $\hat{\mathbf{x}}_0(\mathbf{x}_t, t)$ 
 
To characterize the evolution of semantic structure during the generation process, we introduce a stage-sensitive semantic signal $\mathcal{E}(t)$, which is defined as the semantic alignment score between the predicted clean sample and its corresponding prompt:
\begin{equation}
    \mathcal{E}(t) = \phi\!\left(\hat{\mathbf{x}}_0(\mathbf{x}_t),\,\mathbf{z} \right),
\end{equation}
where $\phi(\cdot,\cdot)$ denotes a semantic-structural similarity 
% (text-image alignment) 
scorer (e.g., the CLIP model~\citep{radford2021learning}),  $\hat{\mathbf{x}}_0(\mathbf{x}_t)$ denotes the clean image predicted from latent $\mathbf{x}_t$ by Eq.~\ref{def: pseudo x0}, and $\mathbf{z}$ is the input prompt. Next, we define the semantic evolution rate between adjacent deoising steps as:
\begin{align}
\label{double_clip}
\Delta \mathcal{E}(t)
&= \mathcal{E}(t) - \mathcal{E}(t-1) \nonumber \\
&= \phi\!\left(\hat{\mathbf{x}}_0(\mathbf{x}_t),\,\mathbf{z} \right)
 - \phi\!\left(\hat{\mathbf{x}}_0(\mathbf{x}_{t-1}),\,\mathbf{z} \right).
\end{align}
We further monitor the signal-to-noise ratio (SNR) to capture noise retention during generation, and define $\gamma(t)=\mathrm{SNR}(t)$.
\begin{figure*}[t]
    \centering
    \includegraphics[width=1\linewidth]{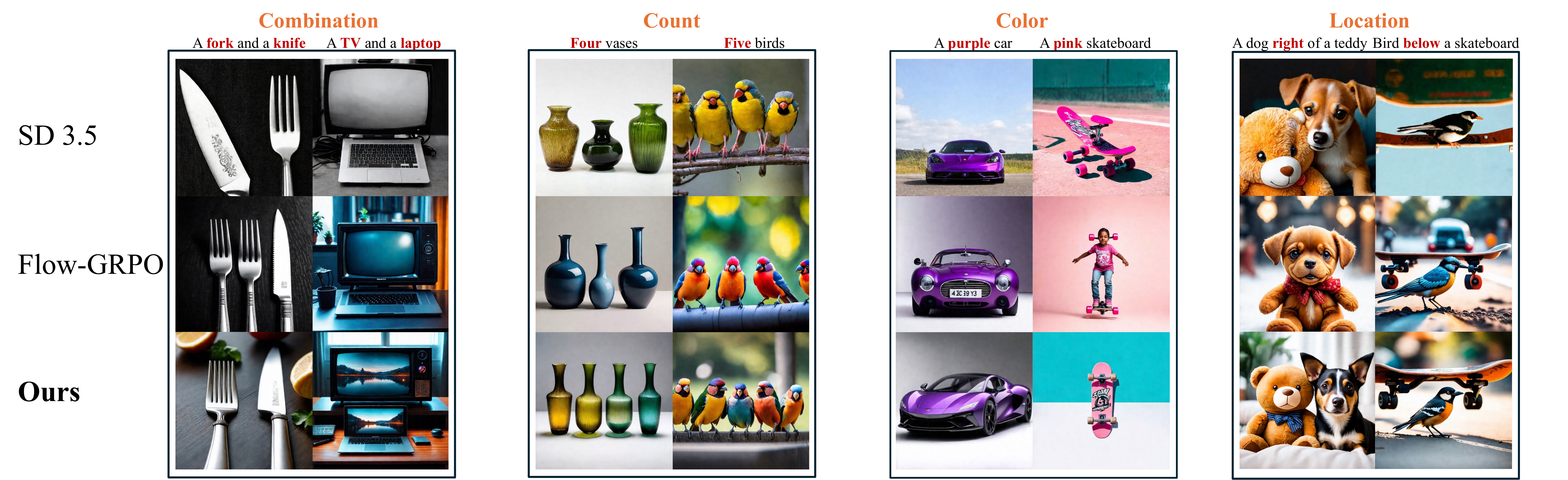}
    \caption{\textbf{Controlled generation}.
    Using prompts with a specific combination, count, color, and location, \textit{SGPO} demonstrates stronger instruction-following ability than baselines.}
    \label{ijcai_color}
\end{figure*}

\paragraph{Adaptive Stage-wise Decomposition.}
As shown in Fig.~\ref{fig:motivation}, when the signal-to-noise ratio stabilizes, indicating a transition from chaotic states to semantic structure formation. This point satisfies $\dfrac{d^2 \gamma(t)}{dt^2} \to 0$ (defined as $t_c^{(1)}$), and can be used to separate the early and middle denoising stages. When the semantic indicator $\Delta \mathcal{E}(t)$ stabilizes, indicating entry into the late stage. This point satisfies $\dfrac{d^2 \Delta \mathcal{E}(t)}{dt^2} \to 0$ (defined as $t_c^{(2)}$), and can be used to separate the middle and late denoising stages. 
$\dfrac{d^2 \gamma(t)}{dt^2} \to 0$ and $\dfrac{d^2 \Delta \mathcal{E}(t)}{dt^2} \to 0$ respectively indicate that the second-order derivatives tend to zero.
We thus define the stage indicator $\pi(t)$:
\begin{equation}
\pi(t)=
\begin{cases}
S_{\text{early}}, 
& 0 \le t < t_c^{(1)}, \\[6pt]
S_{\text{mid}},   
& t_c^{(1)} \le t < t_c^{(2)}, \\[6pt]
S_{\text{late}},  
& t \ge t_c^{(2)} .
\end{cases}.
\end{equation}
This stage indicator provides a basis for dynamically shifting the RL optimization focus across different generation stages.
% % 具体地，我们采用如下形式定义阈值：
% Specifically, we adopt the following expression to define the threshold:
% \begin{equation}
%     \tau_1 = \mu_{\Delta s} + \lambda_1 \sigma_{\Delta s}, 
%     \qquad
%     \tau_2 = \mu_{\Delta s} - \lambda_2 \sigma_{\Delta s},
% \end{equation}
% % 其中 $\mu_{\Delta s}$ 和 $\sigma_{\Delta s}$ 分别表示 $\Delta s_t$ 在去噪轨迹上的运行均值与标准差，$\lambda_1$ 和 $\lambda_2$ 为控制阶段划分灵敏度的超参数。该阶段指示函数为后续在不同生成阶段中动态切换强化学习优化关注点提供了明确而稳定的依据。
% where $\mu_{\Delta s}$ and $\sigma_{\Delta s}$ denote the mean and standard deviation of $\Delta s_t$ over the denoising trajectory, respectively. $\lambda_1$ and $\lambda_2$ are hyperparameters that control the sensitivity of stage division. 

\subsection{Stage-Aware Optimization}
\label{Optimization}
% 基于对 $\mathcal{E}_t$ 的持续感知，我们使强化学习训练过程能够与扩散模型生成过程中语义演化的不同阶段保持一致，并在生成进展的不同区间自适应采用与其生成动力学相匹配的奖励形式。该分阶段奖励设计的核心思想在于：避免在语义尚未成型时即评估偏好奖励质量同时避免在图像已经趋于稳定的阶段对偏好奖励过渡优化，从而实现提高训练精度的同时缓解奖励劫持问题。
% Based on the continuous perception of $\mathcal{E}_t$, we align the RL training process with the distinct stages of semantic evolution during the diffusion model's generation, and adaptively employ reward functions that match the generation dynamics at different stages. The core idea behind this stage-based reward design is to avoid evaluating the quality of preference rewards when semantics have not yet formed, while also preventing excessive optimization of preference rewards once the image has already stabilized. This approach improves training precision while mitigating the reward hacking problem.
By continuously tracking $\Delta \mathcal{E}(t)$ and $\gamma(t)$, we align RL training with the stage-wise semantic evolution of diffusion generation and adapt reward functions accordingly.
This strategy prevents premature reward evaluation in unformed semantics and excessive optimization after stabilization. Finally, it improves optimization precision while reducing reward hacking.

\paragraph{Stage I: Chaotic Stage ($S_{\text{early}}$).}
% 在生成早期，图像处于强噪声主导的混沌状态，缺乏稳定的语义结构，且与末步图像质量之间的关联性极弱。在此阶段，末步偏好奖励与单步去噪动作之间几乎不存在可靠因果关系，直接基于末步奖励进行优化容易将随机噪声扰动误判为有意义的优化信号，从而引入不稳定甚至有害的梯度更新。因此，该阶段的合理优化目标并非直接进行偏好对齐，而是尽快推动生成状态脱离高噪声区域。故本阶段的奖励函数定义为：
% During the early stage of generation, the image is dominated by strong noise and remains chaotic, lacking stable semantic structure and exhibiting only a weak correlation with the final image quality. At this stage, there is almost no reliable causal relationship between the final preference reward and the individual denoising actions. Direct optimization based on the final-step reward can misinterpret random noise as meaningful optimization signals, leading to unstable or even harmful gradient updates. Therefore, the reasonable optimization objective at this stage is to expedite the transition of the generation state away from high-noise regions. Hence, the reward function for this stage is defined as follows:
In the early generation, intense noise dominates the latent state, leading to weak correlation between denoising behaviors and $r\left(\mathbf{x}_0, \mathbf{z}\right)$.
 preference rewards optimization at this stage yields unreliable gradients by attributing noise to meaningful improvement.
Thus, the optimization objective focuses on rapidly exiting high-noise regions, with \textit{the Stage I Rewards Function} defined as:
\begin{equation}
    r_t^{\text{(I)}} \;=\; \lambda_t\bigl\| \mathbf{x}_t - \mathbf{x}_T \bigr\|^2,
\end{equation}
% 其中 $x_T \sim \mathcal{N}(0, I)$ 表示初始噪声状态，$\mathrm{SNR}(t)$ 为与当前时间步 $t$ 对应的信噪比，用于自适应调节不同噪声水平下的优化强度。我们通过最大化当前状态与初始噪声之间的距离鼓励更快远离混沌状态的策略，同时引入信噪比作为时间相关系数，使优化强度与当前噪声水平相匹配。
% where $x_T \sim \mathcal{N}(0, I)$ denotes the initial noise state, and $\mathrm{SNR}(t)$ represents the signal‑to‑noise ratio corresponding to the current timestep $t$, which is used to adaptively adjust the optimization strength under different noise levels. By maximizing the distance between the current state and the initial noise, we encourage the optimization strategy to drive the generation away from chaotic states more rapidly. The introduced signal‑to‑noise ratio serves as a time‑dependent coefficient to align the optimization intensity with the current noise level.
where $\mathbf{x}_T \sim \mathcal{N}(0, I)$ represents the initial noise, and $\lambda_t$ is obtained by scaling $\mathrm{SNR}(t)$ with a constant factor.
Maximizing the distance to the noise origin accelerates departure from chaotic states, while $\mathrm{SNR}(t)$ aligns optimization intensity with the current noise level.

\begin{table*}[t]
    \setlength{\tabcolsep}{5.7pt} % adjust column separation
    \centering
    \caption{\textbf{Reward Hacking Evaluations with Fixed Target Reward and Multiple Cross-Metrics (\textbf{12} in total).} All metrics are obtained with Aesthetic Score (AES) as the targeted reward. Bold indicates the best performance, underlining indicates the second best, and gray entries denote the backbone references that are excluded from comparison. \textit{At matched AES, our method consistently outperforms others in Fidelity, Diversity, and Richness, indicating that it is not overly hacked by targeted reward.}} \label{tab:main}
    \centering
\begin{tabular}{l|c|ccccccccccccc}
    \toprule

    \multirow{2}{*}{Method} & \multicolumn{4}{c}{\textbf{Preference}} & \multicolumn{3}{c}{\textbf{Fidelity}} & \multicolumn{3}{c}{\textbf{Diversity}} & \multicolumn{3}{c}{\textbf{Richness}} & \multirow{2}{*}{\#Top2} \\
    \cmidrule(lr){2-5}\cmidrule(lr){6-8}\cmidrule(lr){9-11}\cmidrule(lr){12-14}
        & \textcolor{gray}{AES}&PS$\uparrow$& IR$\uparrow$ & HPS$\uparrow$ 
        & FID$\downarrow$ & CLIP$\uparrow$ & iFS$\uparrow$ 
        & LPIPS$\uparrow$ & IS$\uparrow$ & TCE$\uparrow$ 
        & BRI$\downarrow$ & NIQE$\downarrow$ & SE$\uparrow$
        &  \\
    \midrule
 \textcolor{gray}{SDv15} & \textcolor{gray}{5.514} & \textcolor{gray}{21.29} & \textcolor{gray}{0.916} & \textcolor{gray}{0.280} & \textcolor{gray}{-} & \textcolor{gray}{0.243} & \textcolor{gray}{-} & \textcolor{gray}{0.646} & \textcolor{gray}{30.31} & \textcolor{gray}{40.23} & \textcolor{gray}{8.02} & \textcolor{gray}{4.131} & \textcolor{gray}{11.34} & -\\
DDPO & \textcolor{gray}{6.004} & 20.99 & 0.897 & 0.266 & 95.15 & 0.239 & 0.494 & 0.601 & 26.17 & 38.98 & 11.01 & 6.967 & 10.98 & 0\\
DPOK & \textcolor{gray}{5.959} & 21.12 & 0.793 & 0.271 & 63.31 & \underline{0.242} & 0.689 & 0.643 & \underline{28.18} & 39.51 & 9.07 & \underline{4.795} & 11.30 & 3 \\
D3PO & \textcolor{gray}{5.970} & 21.10 & 0.829 & {0.276} & 72.37 & 0.240 & 0.696 & 0.640 & 27.97 & \underline{39.95} & 8.69 & 5.003 & 11.37 & 1 \\
B2Diff & \textcolor{gray}{6.029} & 21.09 & \underline{0.965} & \textbf{0.279} & \underline{53.32} & 0.241 & \underline{0.777} & \underline{0.647} & 26.21 & 39.79 & 8.32 & 5.652 & 11.58 & 5 \\
DR & \textcolor{gray}{6.011} & \underline{21.41} & 0.723 & 0.273 & 74.25 & \textbf{0.244} & 0.316 & 0.548 & 27.89 & 39.37 & \underline{7.66} & 5.030 & \underline{11.72} & 4 \\
TDPO & \textcolor{gray}{6.072} & 21.31 & 0.676 & 0.261 & 101.4 & 0.232 & 0.217 & 0.517 & 25.95 & 37.16 & 8.08 & 6.050 & 11.52 & 0 \\
\midrule
\rowcolor{blue!10} \textbf{Ours} & \textcolor{gray}{5.967} & \textbf{21.56} & \textbf{0.998} & \underline{0.278} & \textbf{49.36} & \textbf{0.244} & \textbf{0.868} & \textbf{0.653} & \textbf{29.81} & \textbf{40.14} & \textbf{5.89} & \textbf{4.272} & \textbf{11.86} & \textcolor{red}{\textbf{12}} \\
 \bottomrule
\end{tabular}

    % \caption{Your caption here.}
    % \label{tab:your_label}
\end{table*}

\paragraph{Stage II: Stable Preference Stage ($S_{\text{mid}}$). }

When $\Delta \mathcal{E}(t)$ and $\gamma(t)$ stabilize, the generation process enters the middle stage. In this stage, the semantic representations emerge, and the correlation between $r\left(\mathbf{x}_0, \mathbf{z}\right)$ and the resulting behavior strengthens, making optimization of $r\left(\mathbf{x}_0, \mathbf{z}\right)$ meaningful.
As $\Delta \mathcal{E}(t)$ remains stable, the generation trajectory retains sufficient plasticity. To prevent premature convergence to a single generation mode during RL optimization, we therefore introduce an exploration term based on differences between adjacent latents, modulated by $\Delta \mathcal{E}(t)$. This enables diversity under directional constraints. \textit{The stage II Rewards Function} is defined as:
\begin{equation}
    r_t^{\text{(II)}} \;=\; r\left(\mathbf{x}_0, \mathbf{z}\right)+\; \Delta \mathcal{E}(t)\,\bigl\| \hat{\mathbf{x}}_0(\mathbf{x}_t) - \hat{\mathbf{x}}_0(\mathbf{x}_{t-1}) \bigr\|^2,
\end{equation}
% 其中 $r\left(\mathbf{x}_0, \mathbf{z}\right)$ 表示基于最终生成图像的得到的末步奖励（偏好奖励）。$\hat{\mathbf{x}}_0(\mathbf{x}_t, t)$与$\hat{\mathbf{x}}_0(\mathbf{x}_{t-1}, t-1)$ 表示在第 $t$个时间不和上一步$t-1$ 个去噪步对应的预测干净图像。我们使用欧几里得距离衡量第 $t$个时间不和上一步$t-1$ 时刻图像预测的生成分布上的差异鼓励他们之间的差异越大越好，以此刺激产生跟多样化的生成轨迹。\todor{同时为了使探索不会破坏文图对齐质量，我们使用语义变化率 $\mathcal{E}_t$对探索奖励进行两个调控：1.根据结构变动程度衡量探索力度; 2.由于$\mathcal{E}_t$可正可负，从上一步$t-1$到$t$之间采用的探索策略如果导致文图对齐度下降则根据Eq.~\ref{double_clip} $\mathcal{E}_t$为负探索奖励变为探索惩罚项，可见通过$\mathcal{E}_t$作为探索奖励的系数后，探索奖励在鼓励多样性生成的同时，始终朝着文图一致性高的探索，保证了其合理性。}
% \todor{LL: 两个调控写的不太好，组织下语言}where $r\left(\mathbf{x}_0, \mathbf{z}\right)$ denotes the final-step reward based on the generated image. $\hat{\mathbf{x}}_0(\mathbf{x}_t, t)$ and $\hat{\mathbf{x}}_0(\mathbf{x}_{t-1}, t-1)$ represent the predicted clean images at the current denoising step $t$ and the previous step 
% $t-1$, respectively. We use the Euclidean distance to measure the difference between the predicted generative distributions at step 
% $t$ and step $t-1$, encouraging greater divergence to stimulate more diverse generation trajectories. Meanwhile, to ensure that exploration does not compromise text-image alignment quality, we use the semantic change rate $\mathcal{E}_t$ to modulate the exploration reward in two ways. 
where $r(\mathbf{x}_0, \mathbf{z})$ is the preference rewards.
$\hat{\mathbf{x}}_0(\mathbf{x}_t)$ and $\hat{\mathbf{x}}_0(\mathbf{x}_{t-1})$ are the predicted clean images at steps $t$ and $t-1$, respectively, computed via Eq.~\ref{def: pseudo x0}.
We use their Euclidean distance to quantify differences between successive predicted generative distributions, encouraging controlled deviations to enhance generation diversity.
To ensure that exploration does not compromise text–image alignment, the exploration term is modulated by $\Delta \mathcal{E}(t)$ under two cases: $\Delta \mathcal{E}(t)>0$: The exploration term
    $\; \Delta \mathcal{E}(t)\,\bigl\| \hat{\mathbf{x}}_0(\mathbf{x}_t) - \hat{\mathbf{x}}_0(\mathbf{x}_{t-1}) \bigr\|^2$
    is positive, and the current exploration behavior is encouraged;
$\Delta \mathcal{E}(t)<0$: The exploration term
    $\; \Delta \mathcal{E}(t)\,\bigl\| \hat{\mathbf{x}}_0(\mathbf{x}_t) - \hat{\mathbf{x}}_0(\mathbf{x}_{t-1}) \bigr\|^2$
    is negative, and the current exploration behavior is suppressed.

% 通过引入$\Delta \mathcal{E}_t$作为探索调制项保证了探索始终沿着有有利于语义结构提升的方向进行生成轨迹多样性的探索。
Introducing $\Delta \mathcal{E}(t)$ as a modulation signal ensures that exploration remains directionally constrained toward semantic structure improvement. Thereby, this reward design promotes trajectory diversity while preserving semantic consistency.

\paragraph{Stage III: Convergence Stage ($S_{\text{late}}$).}
% 当相邻去噪步之间的语义变化率 $\Delta s_t$ 逐渐趋近于零时，当前去噪步被判定为生成后期阶段，当生成过程进入后期，图像的全局结构与语义关系已基本确定，语义变化率趋近于零表明生成轨迹逐渐稳定。在这一阶段，继续以提升末步奖励为目标进行强化优化往往收益有限，反而容易引发过度优化，导致生成结果偏离原始数据分布并加剧奖励劫持问题。

% 因此，该阶段的优化目标应从奖励最大化转向稳定收敛。我们通过引入与预训练模型预测结果之间的距离惩罚，并以 $\Delta s_t$ 作为权重，使得当语义变化趋于饱和时，模型被显式拉回到可信的生成分布邻域，从而抑制不必要的策略偏移并保证生成质量的稳定性。

% 即满足

% When the semantic change rate $\Delta s_t$ between adjacent denoising steps gradually approaches zero, the current denoising step is determined to be in the late generation stage. At this point, the global structure and semantic relationships of the image are largely established, which indicates that the generation trajectory is stabilizing. During this stage, continuing to optimize for higher final-step reward via RL typically yields diminishing returns and may instead lead to over‑optimization, causing the generated outputs to deviate from the original data distribution and exacerbating reward hacking.Therefore, the optimization objective in this stage should shift from reward maximization to stable convergence. We introduce a penalty term based on the distance from the pre‑trained model’s prediction, weighted by $\Delta s_t$. When semantic change saturates, the model is explicitly pulled back to the original generation distribution. This strategy suppresses unnecessary policy drift and maintains stable generation quality. 
When $\dfrac{d^2 \Delta \mathcal{E}(t)}{dt^2} \to 0$. Further reward optimization offers negligible benefits and carries the risks of over-optimization. We thus guide stable convergence. In this stage, we define \textit{the Stage III Rewards Function} as
\begin{equation}
    r_t^{\text{(III)}} \;=\; -\frac{1}{2}
\,\bigl\| \hat{\mathbf{x}}_0(\mathbf{x}_t) - x_0^{\mathrm{pre}} \bigr\|^2,
\end{equation}
% 其中 $\hat{x}_0^{\mathrm{pre}}$ 表示预训练扩散模型在相同条件下的预测干净图像。
where $x_0^{\mathrm{pre}}$ denotes the output of the pre‑trained diffusion model under the same $\mathbf{z}$ (prompts). 

\paragraph{Training Reward function.} Based on the above design, the overall reward function is formulated as:
% \begin{align*}
%     \mathcal{L}_{\text{adaptive}} =& - \frac{1}{N} \sum_{i,t}^{N,T} \left[ y_t[i] \log \psi(\mathbf{s}_t[i]) \right.\\
%     & \left. + (1 - y_t[i])\log(1 - \psi(\mathbf{s}_t[i])) \right].
% \end{align*}

% \begin{equation}
% \label{eq:overall_reward_cases}
% R_\text{total}=
% \begin{cases}
% r_t^{\text{(I)}} \;=\; \lambda\bigl\| \mathbf{x}_t - \mathbf{x}_T \bigr\|^2, & \pi(t)=S_{\text{early}},\\
%     r_t^{\text{(II)}} \;=\; r\left(\mathbf{x}_0, \mathbf{z}\right)+\; \Delta \mathcal{E}_t\,\bigl\| \hat{\mathbf{x}}_0(\mathbf{x}_t) - \hat{\mathbf{x}}_0(\mathbf{x}_{t-1}) \bigr\|^2, & \pi(t)=S_{\text{mid}},\\
%     r_t^{\text{(III)}} \;=\; \frac{1}{2}
% \,\bigl\| \hat{\mathbf{x}}_0(\mathbf{x}_t, t) - x_0^{\mathrm{pre}} \bigr\|^2, & \pi(t)=S_{\text{late}}.
% \end{cases}
% \end{equation}
\begin{equation}
\label{eq:overall_reward_cases}
R_{\text{total}}=
\begin{cases}
r_t^{\text{(I)}}, 
& 0 \le t < t_c^{(1)}, \\[3pt]
r_t^{\text{(II)}},   
& t_c^{(1)} \le t < t_c^{(2)}, \\[3pt]
r_t^{\text{(III)}},  
& t \ge t_c^{(2)} .
\end{cases}.
\end{equation}
Using $R_{\text{total}}$, SGPO drives early generation away from chaos, optimizes preference rewards with diversity exploration in the middle stage after semantic formation, and enforces stable convergence in the late stage. $R_{\text{total}}$ ensures that each generation step receives a distinct step-wise reward signal, thereby alleviating reward sparsity.
The use of differentiated rewards further prevents the model from exploiting reward shortcuts caused by repeatedly applying the same reward across steps.

\begin{figure*}[htbp]
    \centering
    \includegraphics[width=1\linewidth]{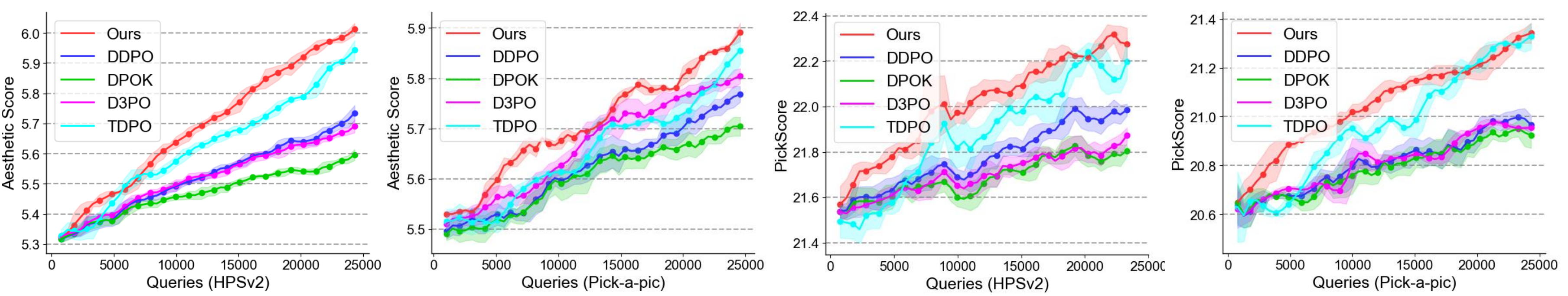}
    \caption{\textbf{Results Trained with Multiple Datasets and Target Rewards}, that are, HPSv2 and Pickapic prompts with Aesthetic Score and PickScore as the optimization objectives. \textit{Notably, TDPO leads to reward hacking manifested as text–image misalignment during reward improvement (Appendix Fig.~\ref{fig:hack_impre}), whereas our method avoids this issue.}}
    \label{fig:main}

\end{figure*}
\begin{figure}[t]
    \centering
    \includegraphics[width=1.0\linewidth]{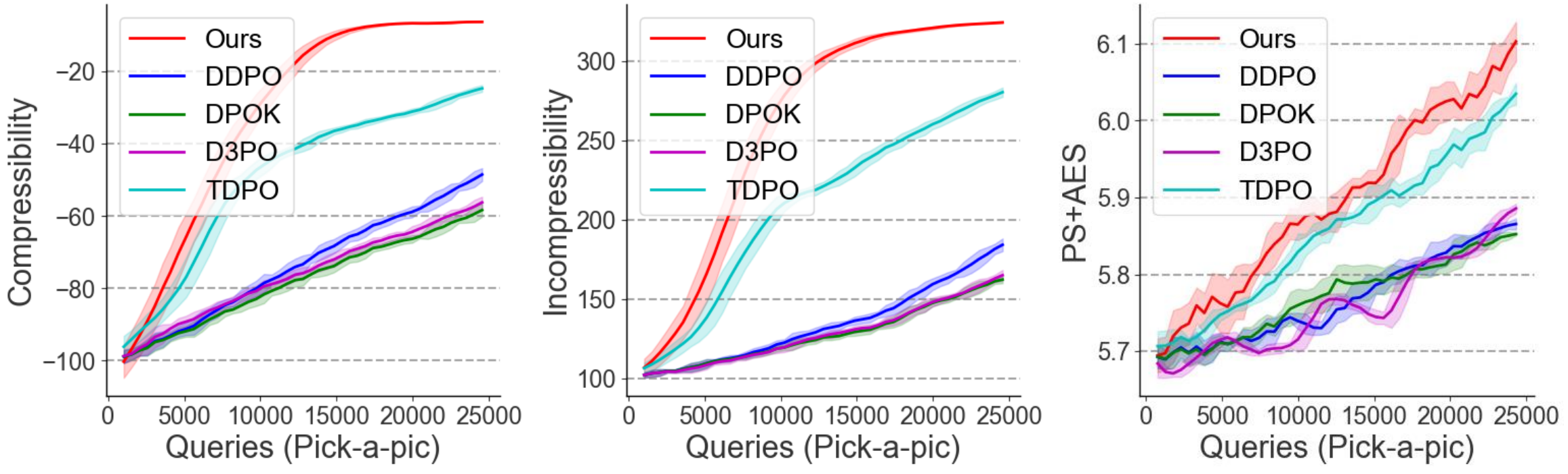}
    \caption{Optimization results on more objectives, including compressibility, incompressibility, and multi-objective (PS+AES).}
    \label{fig:more_obj}
    % 50k
\end{figure}
\begin{figure}[htbp]
    \centering
    \includegraphics[width=1.0\linewidth]{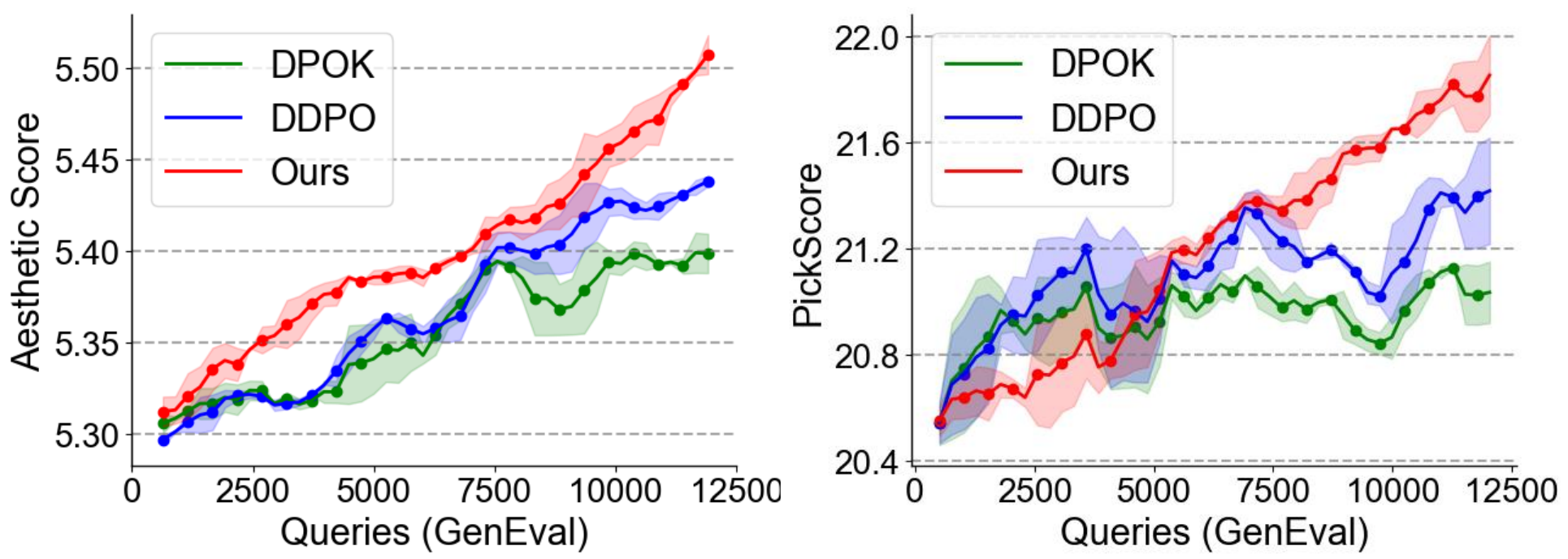}
    \caption{Optimization results on GenEval datasets.}
    \label{fig:geneval}
    % 50k
\end{figure}
\begin{figure}[t]
    \centering
    \includegraphics[width=1\linewidth]{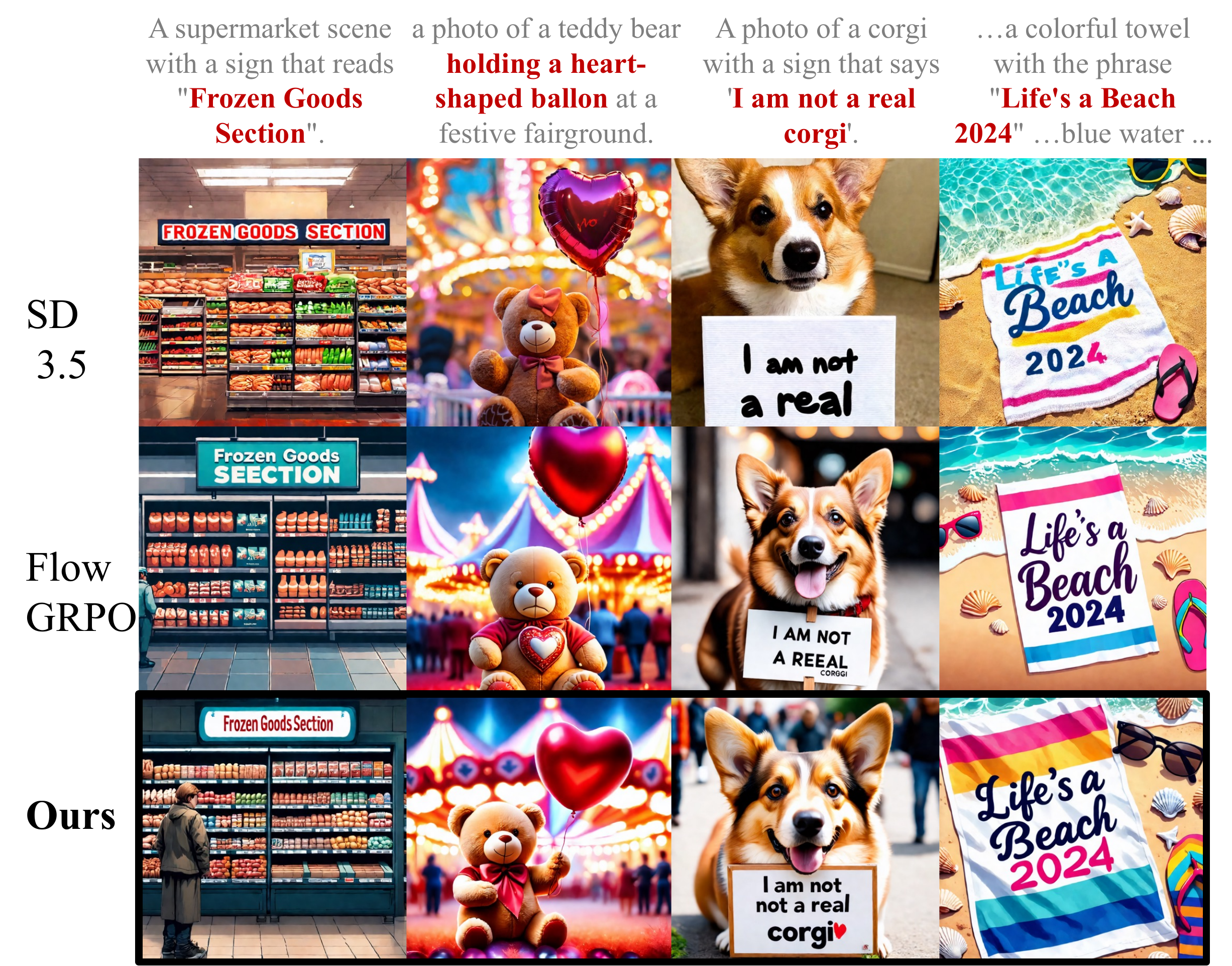}
    \caption{\textbf{Complex prompt alignment}.
    We evaluate alignment under complex prompts and OCR generation, where \textit{SGPO} significantly outperforms baseline methods.}
    \label{fig:alig}
        \vspace{-0.5cm}
\end{figure}
\begin{figure}[h]
        \includegraphics[width=1\linewidth]{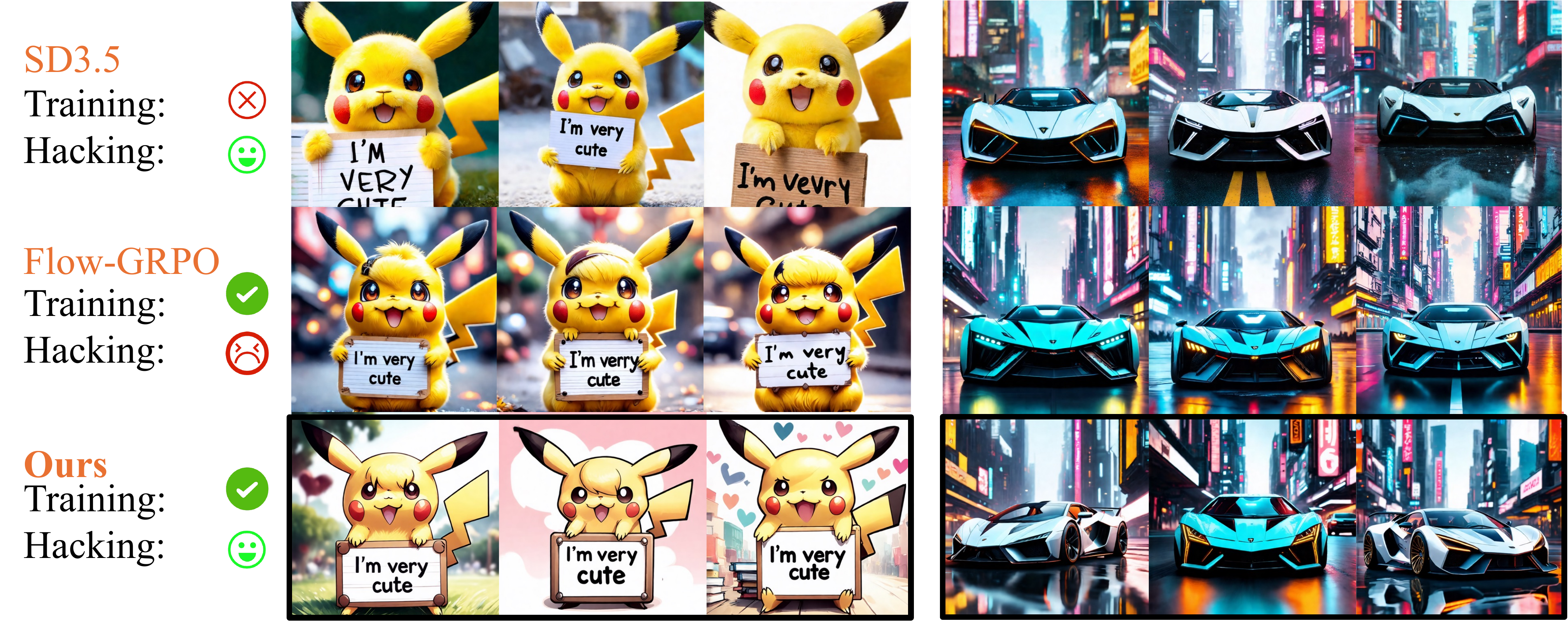}
        \caption{\textbf{Generation diversity comparison}. SGPO yields higher diversity in shape, orientation, and color than baselines under equal training epochs.
        % \textcolor{red}{3.5和数据集}
        }
        \label{fig:diverstiy}
        \vspace{-0.5cm}
\end{figure}

\begin{figure}[ht]
    \centering
\includegraphics[width=1\linewidth]{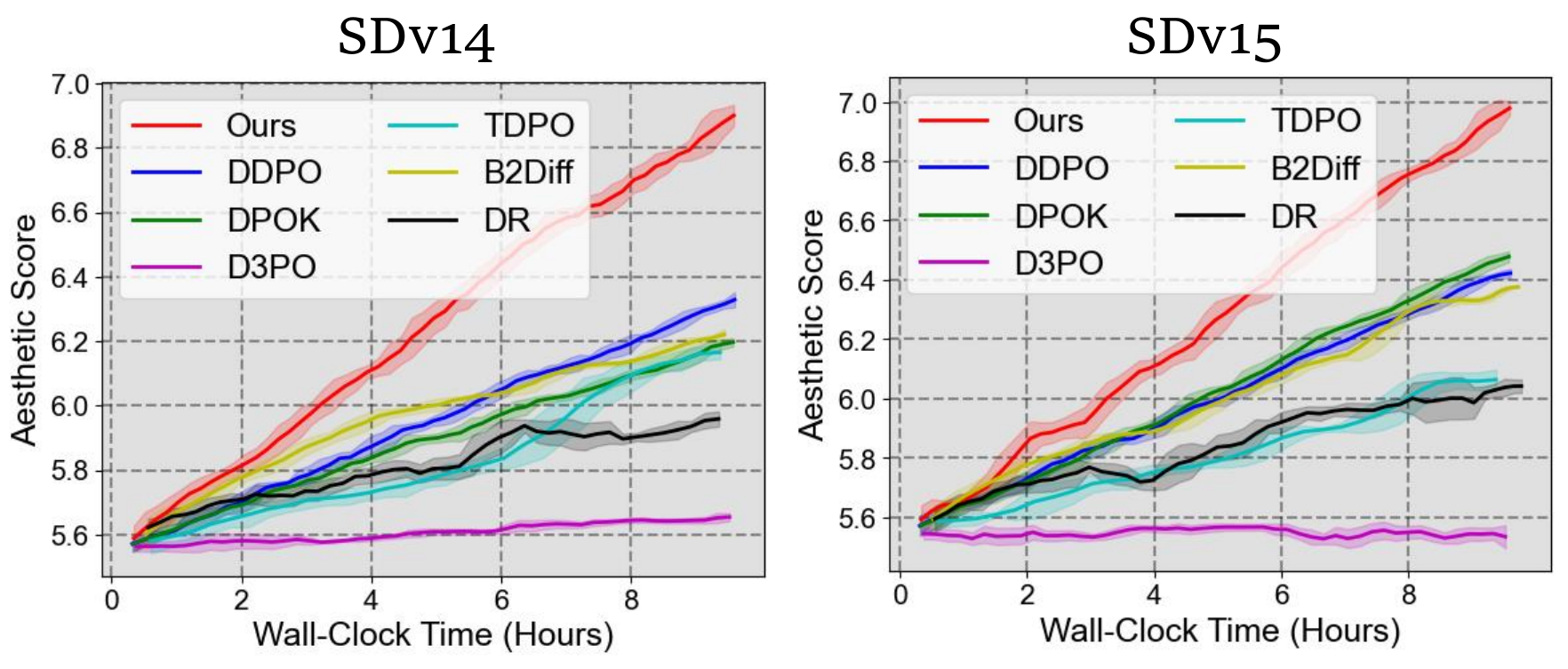}
    \caption{Optimization Efficiency with \textbf{wall-clock time} as the x-axis reference.}
    \label{fig:wallclock}
\end{figure}

\section{Experimental Evaluation}
\subsection{Implementation Details}

\noindent \textbf{Datasets.}
We utilize four popular datasets for the fine-tuning process, that are, \textit{HPSv2}~\citep{hpsv2}, \textit{Pick-a-Pic}~\citep{kirstain2023pick}, \textit{GenEval}~\citep{geneval}, and Simple Animals~\citep{black2023training}. Specifically, we focus on the HPSv2-photo subset and GenEval, which offer a more realistic and natural style. In addition, we select a subset of the Pick-a-Pic validation set comprising 500 prompts, characterized by abstract and surreal imagery.

\noindent \textbf{Rewards and Metrics.}
The training rewards in our approach cover a wide range, including the Aesthetic Score (AES)~\citep{aesthetic}, PickScore (PS)~\citep{kirstain2023pick}, JPEG compressibility, incompressibility, and multi-objective (AES+PS). Then, to quantify the issue of reward hacking, we further introduce metrics for evaluation dimensions:    \textit{Aesthetic Preference:} AES, PS, IR, and HPSv2~\citep{hpsv2}.
    \textit{Image Fidelity:} ClipScore~\citep{clip} (Clip), Fréchet Inception Distance (FID)~\citep{fid}, and improved F1 Score (iFS)~\citep{iPR}.
    \textit{Generative Diversity:} LPIPS~\citep{LPIPS}, TCE~\citep{TCE}, and Inception Score~\citep{inception} (IS).
    \textit{Compositional Richness:} NIQE~\citep{niqe}, BRISQUE (BRI)~\citep{brisque}, and Spectral Entropy (SE)~\citep{spectralE}.\\
\noindent \textbf{Baselines.}
We compare against both traditional \textit{SDE-based} and \textit{flow-matching-based} diffusion baselines.
Due to architectural differences, comparisons are conducted separately within each framework.
The SDE-based baselines include DDPO~\citep{black2023training}, DPOK~\citep{fan2024reinforcement}, D3PO~\citep{yang2024using}, TDPO~\citep{zhang2024confronting}, DenseReward~\citep{densereward} (DR), and B2DiffusionRL~\citep{b2diff} (B2Diff), while the \textit{flow-matching-based} baselines include the flow-matching versions of GRPO (Flow-GRPO), PPO (DDPO), SFT (D3PO), and DPO (Diffusion-DPO). Notably, our method can be integrated with any of these baselines. For simplicity, we refer to the combination of DDPO and our method as `Ours'. The shadows of the curves represent the standard deviations with three seeds. 

\noindent \textbf{Backbones.}
To evaluate robustness across architectures, our evaluation is based on two comprehensive benchmarks: 1) \textit{\textbf{Traditional}} Stable Diffusion models, including v1.4 (SDv1.4), v1.5 (SDv1.5), v2.1-turbo (SDv2.1), and XL 1.0 (SDXL). 2) \textit{\textbf{Flow-matching}}-based models such as SD 3.5-medium (SD3.5) and Flux. The experiments were carried out on a single node with 8 NVIDIA Tesla H20 GPUs.

% TODO: mention we use the same LAION aesthetics predictor as in~\cite{black2023training}

\subsection{Experiments on Pre-defined RL Objectives}
\label{RLobj}
% 在本小节中，我们对比了 \ourmethod{} 与几种新颖的强化学习基线方法在优化预设 RL 目标任务上的表现。基线方法包括采用稀疏奖励的 DDPO、DPOK 和 D3PO，以及针对稀疏奖励改进提出的密集奖励方法 xxx。
In this subsection, we compare the performance of \ourmethod{} with several novel RL baselines for optimizing predefined RL objectives. The baseline methods include DDPO, DPOK, D3PO, and TDPO, which use sparse rewards to guide the optimization process.\\
% involving human
% , as well as the dense reward method xxx proposed as an improvement for sparse rewards.
% 微调过程中奖励变化曲线如图xx所展，我们的方法在美学质量、可压缩性和不可压缩性三方面均显著优于基线方法。在相同的优化步数下，我们的方法奖励曲线上升更快，取得了最高的分数。
% 这些优势得益于\ourmethod{}自适应调节的内部奖励机制，动态平衡探索与利用的比例，提升了生成结果与预设目标的一致性。自适应调节机制在去噪初期侧重多样性与细节扩展，防止图像过于简单，而在后期逐步引导模型优化图像的整体结构，确保其符合真实分布。这种分阶段的调节策略有效地提升了生成图像的质量、细节和结构复杂性。图~\ref{jianbian}展示了我们方法优化过程中的图片变化过程。
\textbf{Query-wise learning performance.} Fig.~\ref{fig:main} shows the curves of rewards on different downstream objectives during the fine-tuning process, where we align each method via query number. The results indicate that our method significantly outperforms baseline methods in terms of aesthetic quality and human preference (PickScore) across two popular datasets, i.e., Pick-a-Pic and HPSv2. Our method's reward curve rises faster with the same optimization steps, achieving the highest scores. Then, we conduct experiments on more objectives and datasets, including Compressibility, Incompressibility, and multi-objective as rewards, and GenEval~\citep{geneval} as the dataset. The results shown in Fig.~\ref{fig:more_obj} and Fig.~\ref{fig:geneval} further demonstrate the superior reward-learning performance of our method. \\
% These advantages are due to the adaptive internal reward mechanism of \ourmethod{}, which dynamically balances the exploration-exploitation ratio, enhancing the alignment of generated results with preset objectives. The adaptive mechanism emphasizes diversity and detail expansion in the early de-noising stages, preventing overly simplistic images while gradually guiding the model to optimize the overall structure of the images in later stages to ensure consistency with the actual distribution. This phased adjustment strategy improves the generated images' quality, detail, and structural complexity. Figure~\ref{jianbian} shows the changes in images throughout our optimization process.\\
\textbf{Wall-clock time learning efficiency.} To evaluate the real-world application efficiency, we aligned each method with wall-clock time.
In Fig.~\ref{fig:wallclock}, we further consider two advanced baselines, i.e., DenseReward (DR) and B2DiffusionRL (B2Diff), with SDv14 and SDv15 as backbones. Despite deploying our method may slightly introduce further time consumption for each-query-optimization, the overall wall-clock efficiency is still significantly superior to the baselines.

\subsection{Evaluations on Mitigating Reward Hacking}
\textbf{Multi-metrics Results under Fixed Target Reward.}
To enable a fair comparison of reward hacking among different approaches, all methods were trained until their aesthetics scores approached 6.0, after which their performance was evaluated using the remaining metrics. As reported in Tab.~\ref{tab:main}, we additionally include the original SDv1.5 model as a reference point. The results show that, when AES is adopted as the optimization objective, our approach consistently achieves top-two performance across all evaluation criteria.  This demonstrates that \ourmethod{} avoids overfitting to the training reward and maintains strong generalization across evaluation metrics.\\
\textbf{Results under multi-intensities Target Reward.}
In addition to assessing final generation quality, we systematically analyze reward hacking effects across various reward learning regimes for different baseline models. Fig.~\ref{fig:hacking} demonstrates that our method is largely robust to reward hacking when AES values are below 7. Although minor performance degradation appears during the late optimization stage, which can be attributed to inevitable over-optimization of the reward, our method maintains superior performance compared to all baselines throughout training.

\subsection{Ablation study}
\label{ablation}

Since we include three stages with individual reward design based on the denoising nature of diffusion models, in Tab.~\ref{tab:variant_comparison}, we introduce Ablation Variants Var1 to Var4 to discuss the influence of reward for each stage (\textit{i.e.}, S1, S2, S3). We also introduce a fixed strategy that implements fixed exploration reward at every stage (Fixed). It can be observed that the reward for the chaotic stage (S1) can slightly enhance the reward learning performance. The reward for the stable performance optimization stage contributes the most to reward learning and generative diversity, while still being less promising in semantic alignment. The reward for the final convergence stage improves the alignment performance while maintaining the other metrics.
Our results underscore not only the necessity of each proposed component but also the remarkable performance that emerges from their integration.

\begin{figure}[t]
    \centering
    \includegraphics[width=1\linewidth]{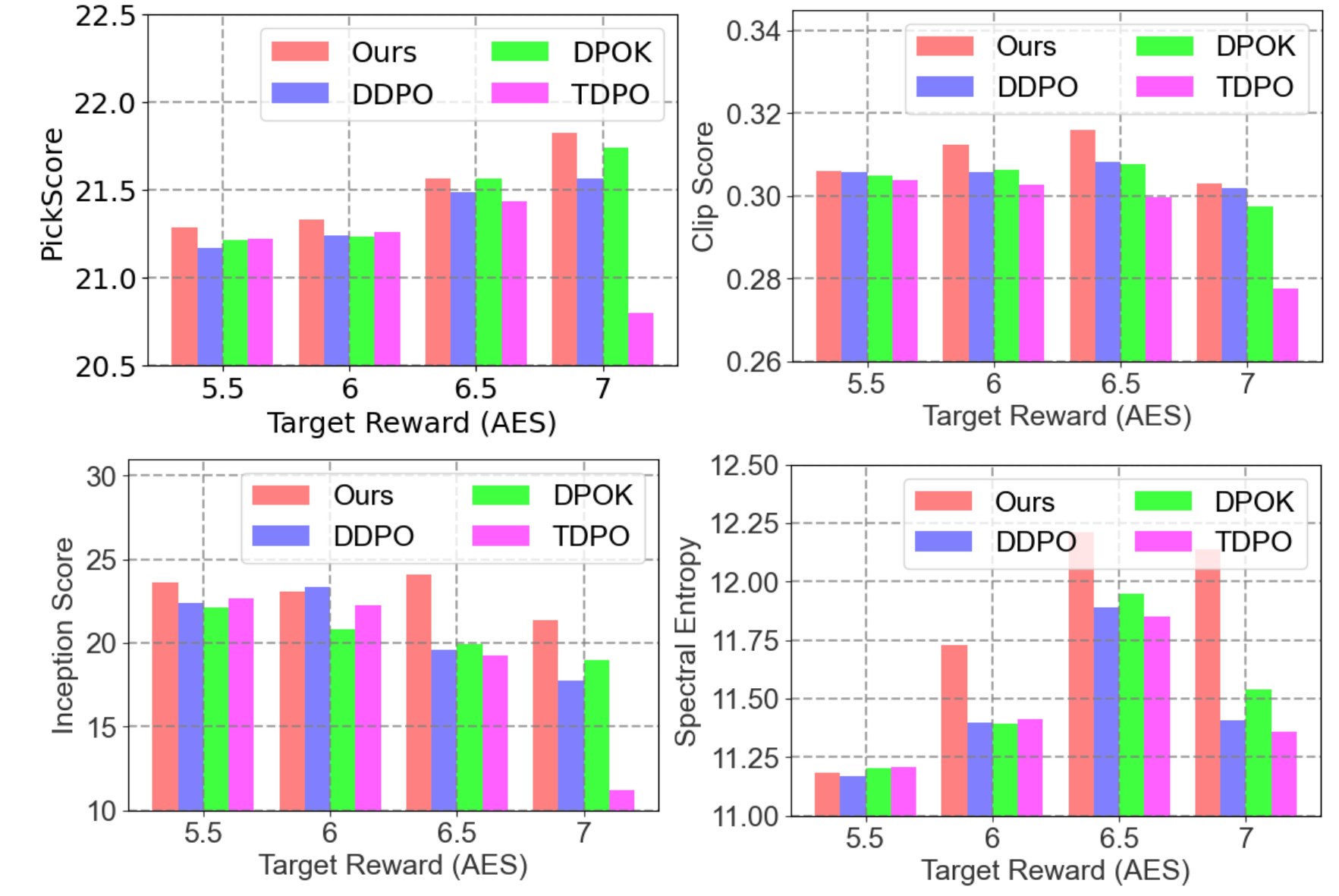}
    \caption{Evaluation on reward hacking with multi-intensities of Target Reward (AES).}
    \label{fig:hacking}
\end{figure}
\begin{table}[t]

\setlength{\tabcolsep}{2.9pt} % 可根据需求微调
\centering
\caption{\textbf{Ablation Study on Each Proposed Component.} The ablation variants across several metrics on \textit{HPSv2 Dataset}, with PS as the target reward. `Top2' indicates second-best performance. Time refers to the training time (hours) to reach PS=22.2.}
\label{tab:variant_comparison}
\begin{tabular}{l|ccc|cccccc}
\toprule
Variants &    S1&S2&S3&AES$\uparrow$ & CLIP$\uparrow$ & IS$\uparrow$ & SE$\uparrow$ & Time$\downarrow$  & Top2 \\
\midrule
Base     &    -&-&-&5.310 & 0.239 & 21.91 & 11.01 & 10.41 &0\\
 Fixed& -& -& -& 5.674& 0.241& 22.59& 11.73& 8.02&0\\
\midrule
Var1
&    \checkmark&&&5.353& 0.238& 22.04& 10.98& 9.27&0\\ 
Var2
&    &\checkmark&&5.732& 0.239& 22.61& 11.32& 8.95 &0 \\
Var3&    \checkmark&\checkmark&&\underline{5.765}& 0.241& \underline{23.49}& 11.20& \underline{6.81}&3 \\
 Var4& & \checkmark& \checkmark& 5.730& \textbf{0.244}& 22.19& \underline{11.92}& 7.69&2\\ \midrule
\rowcolor{blue!10}\textbf{Ours}    &    \checkmark&\checkmark&\checkmark&\textbf{5.808}& \textbf{0.244}& \textbf{23.92}&\textbf{ 11.97}& \textbf{6.59}&5 \\
\bottomrule
\end{tabular}
\end{table}

\begin{figure}[htbp]
    \centering
    \includegraphics[width=1\linewidth]{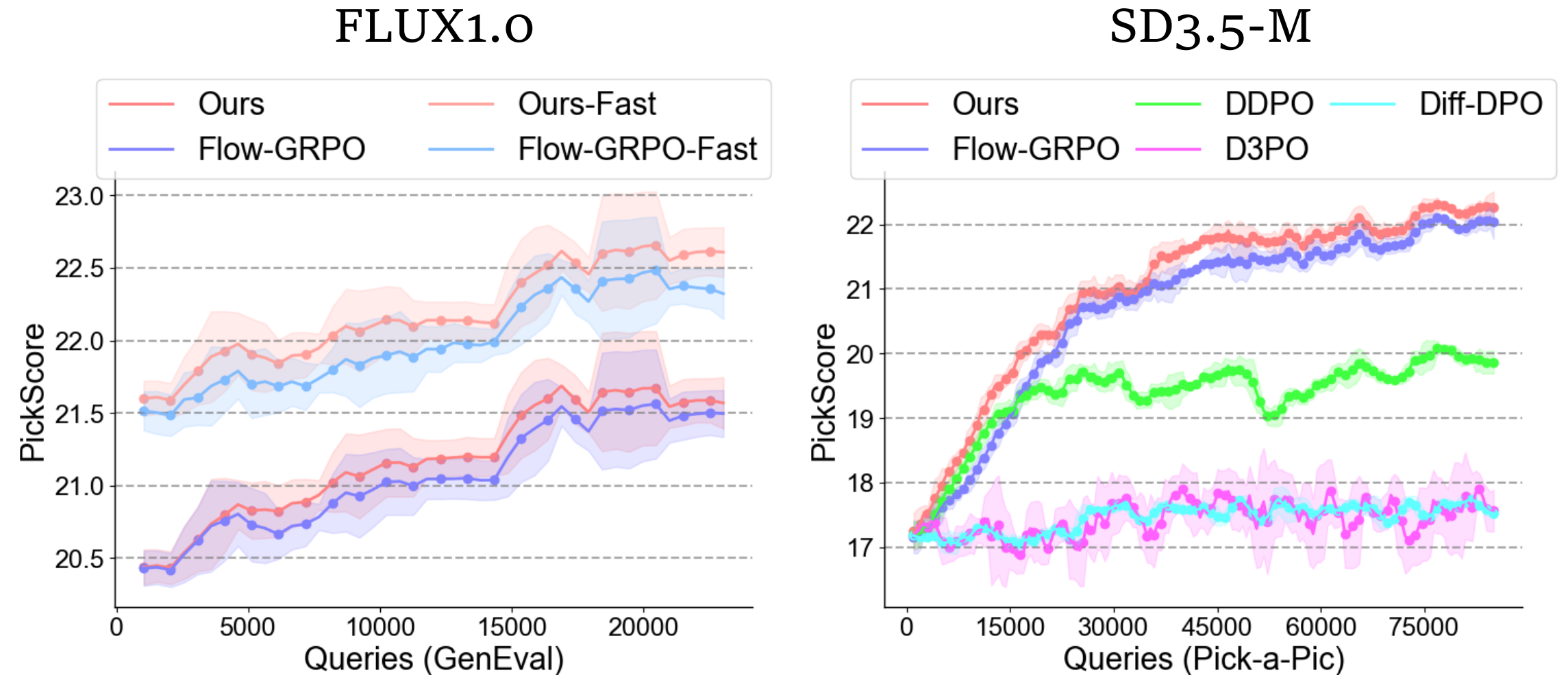}
    \caption{Results on the advanced \textbf{flow-matching-based} diffusion models comparing with different baselines.}
    \label{fig:flow}
\end{figure}
\begin{figure}[htbp]
    \centering
    \includegraphics[width=1\linewidth]{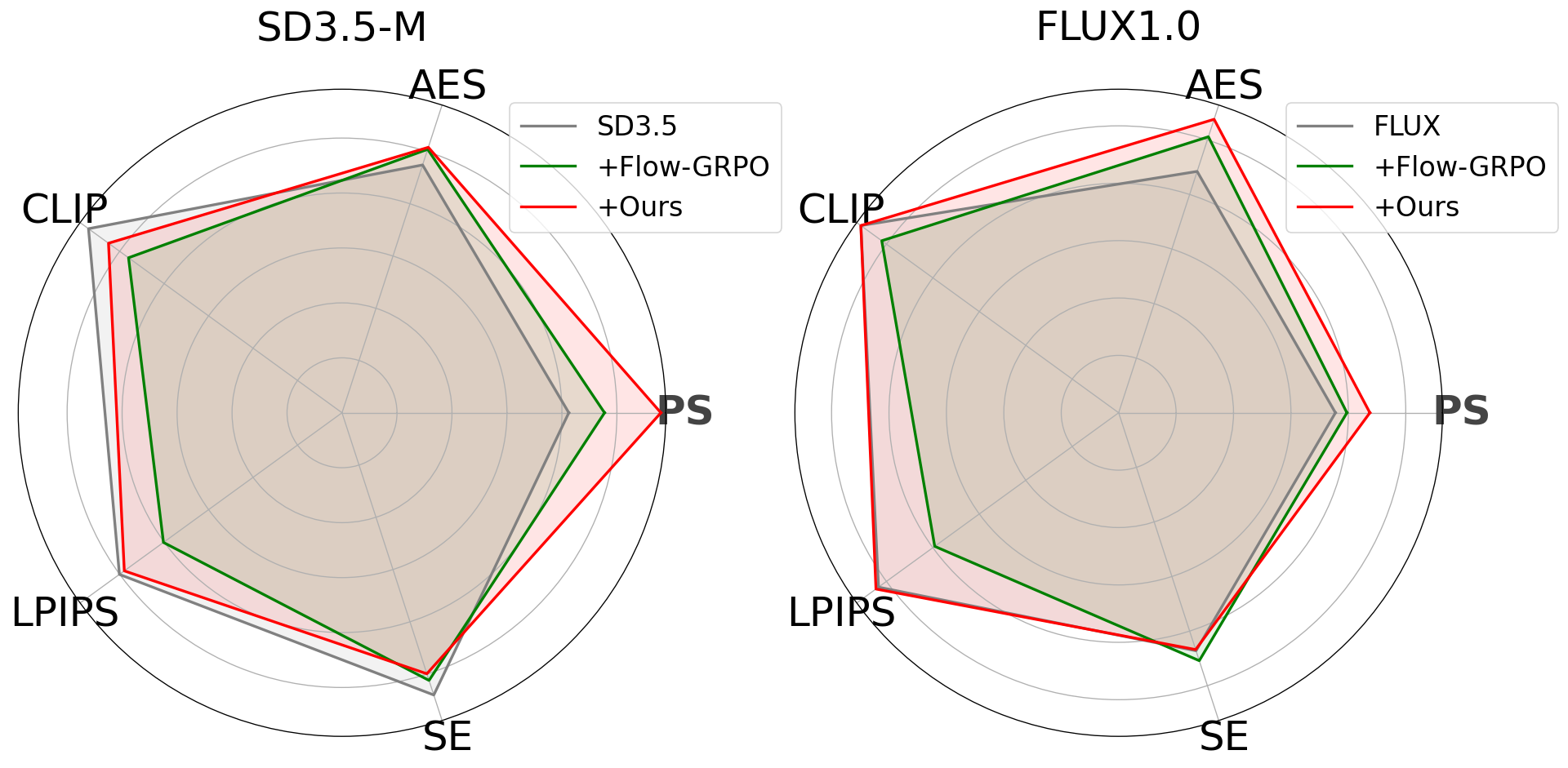}
    \caption{ \textbf{Reward Learning and Hacking Performance on SD3.5 and FLUX}. All baselines are trained with the same 240 epochs and evaluated on the unseen test set. \textit{PS is treated as the training reward for optimization (marked as gray)}. The remaining metrics are evaluated after training.}
\label{fig:hack_sd35_flux}
\end{figure}
\begin{figure}[htbp]
    \centering
    \includegraphics[width=1\linewidth]{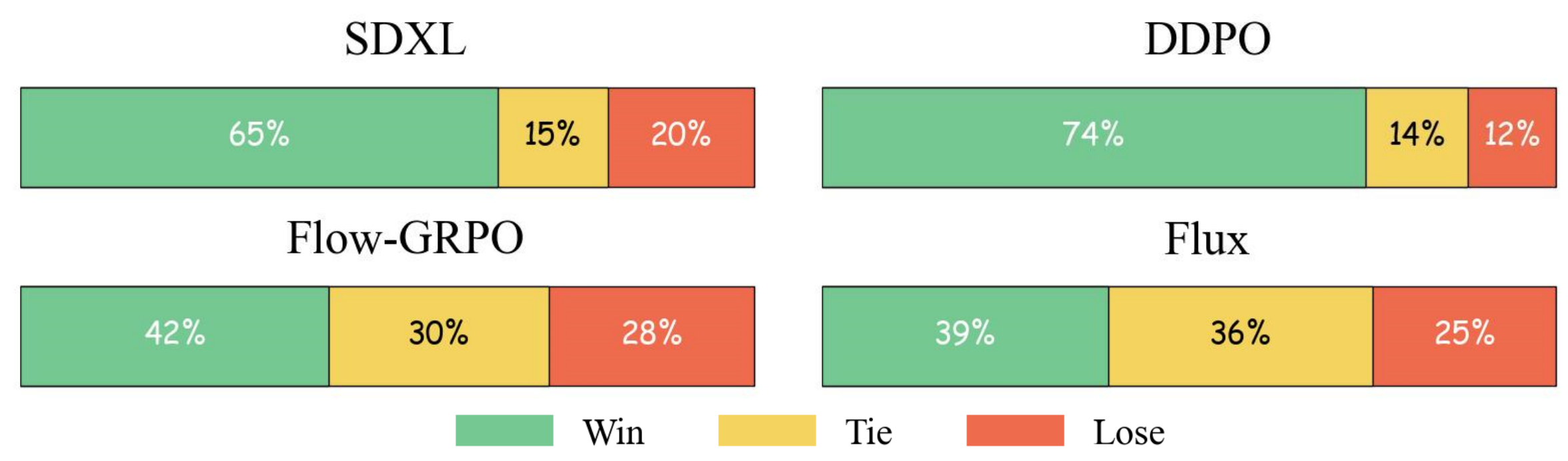}
    \caption{User study of our method versus other baselines.}
    \label{fig:user}
\end{figure}

\subsection{Results on Flow-Matching-Based Models}
As shown in Fig.~\ref{fig:flow} and Fig.~\ref{fig:hack_sd35_flux}, we implement our method for the advanced flow-matching-based diffusion models, namely FLUX and SD3.5. We reproduce supervised fine-tuning as D3PO, Flow-PPO as DDPO, and Flow-DPO as Diffusion-DPO (Diff-DPO). As shown in Fig.~\ref{fig:flow}, we successfully enhance the performance of Flow-GRPO, achieving the SoTA reward learning performance on both models. Notably, we then investigate the performance of the reward hacking issue by evaluating on the unseen set. As shown in Fig.~\ref{fig:hack_sd35_flux}, we train all models on pick-a-pic train set with 360 epochs and evaluate on the test set. Our results can effectively maintain the semantic alignment and generative diversity compared to the backbone results, while surpassing the baseline Flow-GRPO in all five metrics.
% \subsection{Experiments on diversity improvement}
\subsection{Qualitative Evaluation}
\textbf{Generation diversity comparison.}
In Fig.~\ref{fig:diverstiy}, SGPO yields more diverse samples than the baselines under the same training budget. The improvement can be observed in both global attributes (e.g., object shape and viewpoint), and local details (e.g., color and texture). In contrast, the baselines often collapse to a few repeated high-reward patterns, resulting in similar backgrounds and subject layouts, which are consistent with the quantitative diversity results.\\
\textbf{Controlled generation.}
Fig.~\ref{ijcai_color} compares controllability on prompts involving object category, quantity, color, and spatial position. SGPO follows these compositional constraints more reliably, while the baselines often satisfy only part of the prompt. Typical failure cases include correct object types with incorrect counts, or correct colors with mismatched positions.\\
\textbf{Complex prompt alignment and OCR generation.}
In Fig.~\ref{fig:alig}, SGPO performs better on long\&complex prompts that require preserving entity relations, scene semantics, and fine-grained textures. Compared with the baselines, it shows less semantic drift and produces more coherent layouts. In Text-related cases, SGPO also generates clearer and more contextually appropriate text content.

\subsection{User Study}
To subjectively evaluate the visual quality of the generated images, we recruited human subjects to assess outputs from our method and various baselines, all generated with the same prompt and random seed. We employed the Pick-a-Pic dataset and adopted PickScore as the target reward metric. For the baselines, we considered the original FLUX 1.0, the original SDXL, FLUX 1.0 fine-tuned with Flow-GRPO, and SDXL fine-tuned with DDPO. As shown in Fig.~\ref{fig:user}, our method consistently achieved a high win rate across all four settings, demonstrating the superiority of the images generated by our approach in terms of perceived visual quality.

\section{Conclusion}
% 本文首次在RL微调中引入内在动机，解决了现有模型中因稀疏和延迟奖励导致的优化不足和奖励劫持问题。
% 本文提出\ourmethod{}方法，设计了一种独立于下游任务的通用内在奖励构建机制，从细粒度层面鼓励探索以加速去噪优化和增加生成多样性。此外，为实现了在RL微调扩散模型中生成质量与多样性和最大化预设任务目标之间的动态平衡，本文提出内部奖励的自适应调节机制，该机制基于图像去噪程度实时调整内在奖励的强度与应用时机，从而科学平衡探索与利用。实验结果表明，本文方法在最大化美学、压缩性和对齐性等RL目标的同时，显著提升了生成的多样性和正畸效果，表现远超现有方法。
% 本文提出的自适应内在奖励构建方法为未来基于RL微调的扩散模型优化提供了新思路。
% This paper addresses the optimization deficiencies and reward hijacking issues caused by sparse and delayed rewards in RL fine-tuning diffusion models. We propose \ourmethod{} that introduces intrinsic motivation into intermediate steps of de-noising policy optimization for the first time. To effectively complement various extrinsic rewards, we design a task-agnostic intrinsic reward construction mechanism that encourages fine-grained exploration to accelerate de-noising optimization and enhance generation diversity. Moreover, we propose an adaptive intrinsic reward adjustment mechanism to balance generation quality, diversity, and preset task objectives in RL fine-tuning. This mechanism dynamically adjusts the intensity and timing of intrinsic rewards in real time based on the degree of image de-noising, achieving a proper balance between exploration and exploitation. Experimental results demonstrate that our method significantly enhances diversity and corrective effects while maximizing RL objectives such as aesthetics, compressibility, and incompressibility, outperforming existing approaches. The adaptive intrinsic reward construction method proposed in this paper offers new perspectives for future RL fine-tuning optimization of diffusion models.
In this paper, we propose Stage-Guided Per-step Optimization (SGPO), a stage-aware RL framework that aligns optimization objectives with the dynamics of diffusion generation, i.e., chaotic stage, stable preference stage, and convergence stage. SGPO effectively improves alignment, preserves diversity, and alleviates reward hacking. Experimental results demonstrate that respecting the stage-wise structure of diffusion processes is crucial for robust and efficient RL fine-tuning.

\clearpage

\bibliographystyle{ACM-Reference-Format}

\bibliography{refer}

\clearpage

% \clearpage
% \setcounter{page}{1}
% \maketitlesupplementary
% \setcounter{section}{0}

% ----- Supplementary -----
\clearpage
\setcounter{tocdepth}{2}     % 让 section/subsection 放入目录
\appendix
\section*{Appendix}

\section{Discussion of More Related Work}
\label{related_work_add}

\begin{figure}
    \centering
    \includegraphics[width=1\linewidth]{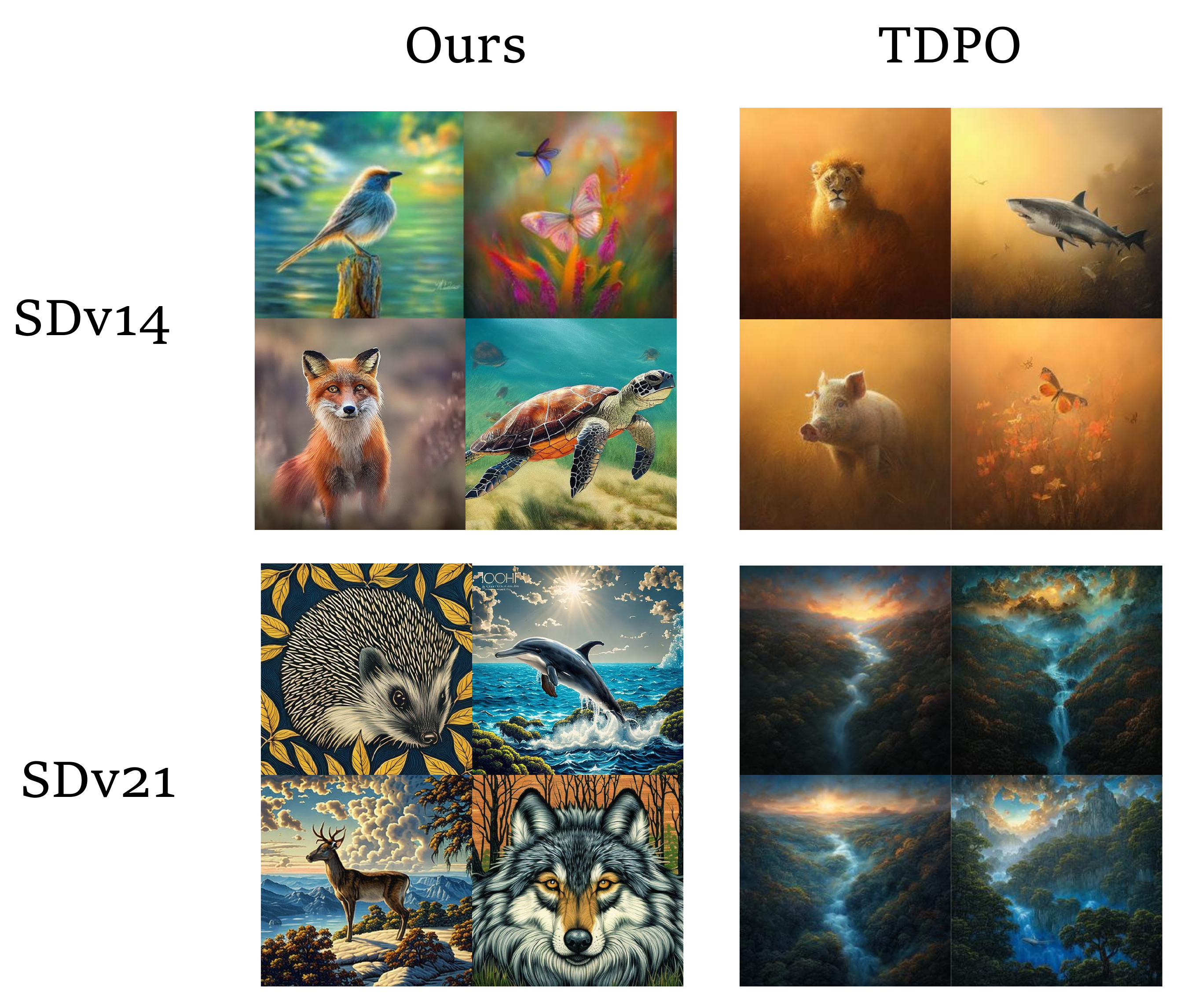}
    \caption{Visual Impression of Reward Hacking. With similar reward scores, the generated images can still be diversified and aligned (Ours), while the reward hacking leads to over-fitting and significantly undermines the semantic information and diversity (TDPO).}
    \label{fig:hack_impre}
\end{figure}

\begin{figure}[t]
        \centering
        \includegraphics[width=1\linewidth]{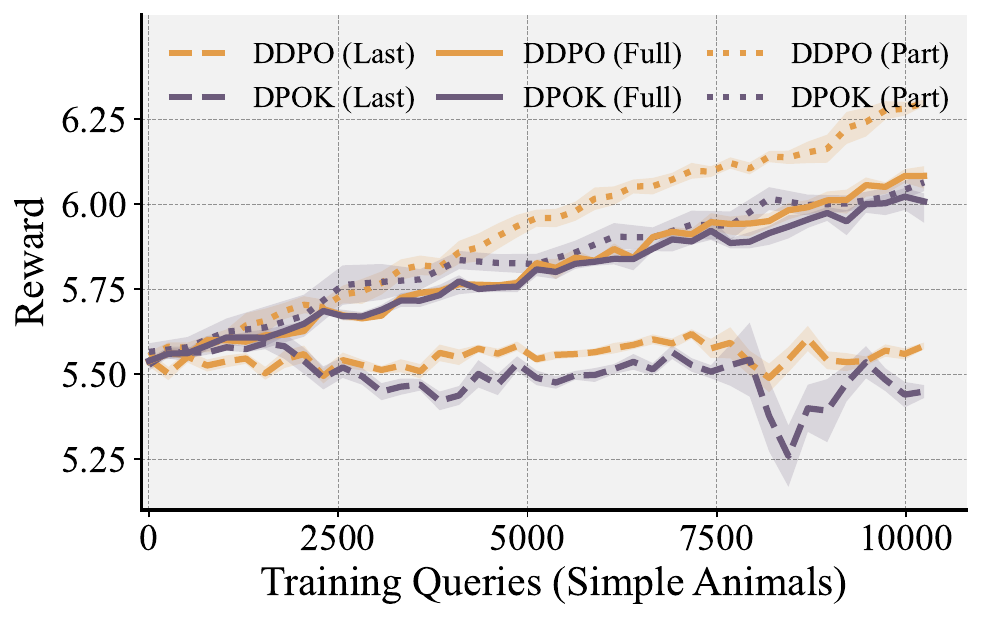}
        \caption{The comparing results of Backpropagating final reward to 1) \textbf{Last}: the last step. 2) \textbf{Full}: full trajectory (widely adopted by existing methods). 3) \textbf{Part}: only the stable middle stage.}
        \label{fig:rewardback}
\end{figure}%
\subsection{Reinforcement Learning Fine-tuning for Diffusion Models}
Recent advances have explored reinforcement learning~\citep{kaelbling1996reinforcement,chaudhari2025rlhf,zhang2025survey,tang2025deep,pippas2025evolution,balhara2025survey,michailidis2025reinforcement} as an effective paradigm for fine-tuning diffusion models toward specific objectives, such as aesthetic quality, human preference alignment, or task-specific constraints. Unlike standard diffusion training, which optimizes likelihood-based objectives, RL-based fine-tuning enables direct optimization of reward functions~\citep{lamba2025alignment,wagenmaker2025steering,zhao2022alphaholdem,gan2024reflective,schulman2017proximal,sampa2026reinforcement}.

Most existing approaches formulate diffusion sampling as a sequential decision-making process, where each denoising step corresponds to an action, and a preference model or evaluator assigns a reward to the final generated sample. To address the extreme reward sparsity inherent in this setup, these methods typically propagate the final-step reward backward to all denoising steps and apply policy gradient updates across the entire trajectory. This strategy has been shown to improve alignment performance in practice.

However, such approaches implicitly treat all denoising steps as equally relevant for reward optimization, overlooking the intrinsic stage-wise structure of the diffusion generation process. As discussed in recent studies~\citep{choi2022perception,li2023autodiffusion,yi2024towards,xie2025dymo}, this simplification can lead to optimization inefficiencies, reward hacking, and degradation of sample diversity, motivating the need for more structured and dynamics-aware reinforcement learning formulations.

\subsection{Reward Hacking}
\label{reward_hacking_appendix}
Reward hacking~\citep{eisenstein2023helping,skalse2022defining,pan2024feedback,laidlaw2024correlated,hadfield2017inverse} is a well-known failure mode in reinforcement learning, where an agent maximizes the reward function through unintended shortcuts rather than achieving the intended task objective. This phenomenon typically arises when the reward function is misaligned with the true optimization goal, leading the policy to exploit spurious correlations rather than learn meaningful behaviors. In Fig.~\ref{fig:hack_impre}, we show the typical hacking impressions on a simple animal prompt set after reward optimization.

In the context of diffusion model fine-tuning with reinforcement learning, reward hacking becomes particularly pronounced due to the sparse and delayed nature of preference rewards~\citep{fan2024reinforcement}. Most existing approaches compute a single reward only at the final generation step and propagate it backward to all denoising steps. This reward-backpropagation strategy implicitly assumes that all intermediate denoising states are equally informative and can be guided by a single reward signal. However, such an assumption ignores the intrinsic stage-wise dynamics of diffusion generation.

As a result, repeatedly optimizing an identical reward across consecutive denoising steps can encourage the model to discover shortcuts that improve the accumulated reward over time without genuinely improving generation quality. These shortcuts often manifest as mode collapse, repetitive structures, or over-amplified local patterns, leading to reward hacking and degradation of generative diversity. Reward hacking in diffusion model RL fine-tuning is not merely caused by reward sparsity, but fundamentally rooted in the temporal homogenization of rewards induced by indiscriminate reward backpropagation~\citep{zhai2025mira,clark2023directly,rafailov2024scaling,bai2025dragon}.

\subsection{Proof}
\textbf{Theorem 3.1}
\textit{If the diffusion model is perfectly trained and the reverse generation follows the score-based SDE,  $\hat{\mathbf{x}}^t_0$ denotes the posterior distribution of $\mathbf{x}_0$ at time $t$, then $\operatorname{Var}(\hat{\mathbf{x}}^t_0 \mid \mathbf{x}_t)$ is a decreasing function during the generation process.}

\begin{proof}
    We consider the forward diffusion process. By the reverse-time diffusion equation models by Anderson, when the diffusion model is perfectly trained, the joint distribution of the forward process is identical to that of the reverse generative process. Therefore, we can equivalently analyze the posterior variance in the generative process using quantities from the forward process.

Consider the forward diffusion process defined as:
\begin{equation}
	\mathbf{x}_t = \sqrt{\bar{\alpha}_t}\,\mathbf{x}_0 + \sqrt{1-\bar{\alpha}_t}\,\boldsymbol{\epsilon}_t, \qquad \boldsymbol{\epsilon}_t \sim \mathcal{N}(\mathbf{0}, \mathbf{I}),
\end{equation}
where $\bar{\alpha}_t = \prod_{k=1}^t (1-\beta_k)$ is the cumulative decay coefficient, which decreases as $t$ increases. Assume the prior distribution is Gaussian:
\begin{equation*}
	p(\mathbf{x}_0) = \mathcal{N}(\mathbf{0}, \sigma_0^2 \mathbf{I}),
\end{equation*}
The conditional distribution is given by the forward process:
\begin{equation*}
	p(\mathbf{x}_t \mid \mathbf{x}_0) = \mathcal{N}\!\left(\mathbf{x}_t; \sqrt{\bar{\alpha}_t}\,\mathbf{x}_0,\; (1-\bar{\alpha}_t)\mathbf{I}\right).
\end{equation*}

By Bayes' theorem, the posterior distribution is proportional to the product of the prior and the likelihood:
\begin{equation*}
	p(\mathbf{x}_0 \mid \mathbf{x}_t) \propto p(\mathbf{x}_0)\, p(\mathbf{x}_t \mid \mathbf{x}_0).
\end{equation*}

Thus, the posterior covariance matrix is
\begin{equation*}
	\boldsymbol{\Sigma}_{\text{post}}(t) = \left( \frac{1}{\sigma_0^2}\mathbf{I} + \frac{\bar{\alpha}_t}{1-\bar{\alpha}_t}\mathbf{I} \right)^{-1}
	= \frac{1}{\frac{1}{\sigma_0^2} + \frac{\bar{\alpha}_t}{1-\bar{\alpha}_t}}\,\mathbf{I}.
\end{equation*}
Let $\sigma_{\text{post}}^2(t)$ denote the scalar variance of the posterior (independent and identically distributed across components); then
\begin{equation}\label{eqe1}
	\sigma_{\text{post}}^2(t) = \frac{1}{\frac{1}{\sigma_0^2} + \frac{\bar{\alpha}_t}{1-\bar{\alpha}_t}} = \frac{1-\bar{\alpha}_t}{\frac{1-\bar{\alpha}_t}{\sigma_0^2} + \bar{\alpha}_t}.
\end{equation}

Since $\bar{\alpha}_t$ is a strictly decreasing function of $t$, $\bar{\alpha}_t$ decreases as $t$ increases. In the \textbf{generative process} (reverse process), time $t$ decreases from $T$ to $0$, i.e., $t$ becomes smaller. We need to examine the trend of $\sigma_{\text{post}}^2(t)$ as $t$ decreases (the generation proceeds).

Let $u = \bar{\alpha}_t$; then $u \in (0,1]$ and decreases as $t$ increases (i.e., increases as generation proceeds). Rewrite Eqn.(\ref{eqe1}) as a function of $u$:
\begin{equation*}
	f(u) = \frac{1-u}{\frac{1-u}{\sigma_0^2} + u} = \frac{1-u}{u + \frac{1-u}{\sigma_0^2}}.
\end{equation*}
Compute the derivative of $f(u)$ with respect to $u$:
\begin{equation*}
	f'(u) = \frac{-\bigl(u + \frac{1-u}{\sigma_0^2}\bigr) - (1-u)\bigl(1 - \frac{1}{\sigma_0^2}\bigr)}{\bigl(u + \frac{1-u}{\sigma_0^2}\bigr)^2} = -\frac{1}{\bigl(u + \frac{1-u}{\sigma_0^2}\bigr)^2} < 0.
\end{equation*}
This shows that $f(u)$ is strictly decreasing in $u$. Since during the generative process $u = \bar{\alpha}_t$ gradually increases (because $t$ decreases, $\bar{\alpha}_t$ increases), $f(u)$ strictly decreases as generation proceeds. That is,
\begin{equation*}
	\frac{d}{dt_{\text{gen}}} \sigma_{\text{post}}^2(t) < 0,
\end{equation*}
where $t_{\text{gen}}$ denotes the generation progress (from $T$ to $0$). Therefore, the posterior variance $\operatorname{Var}(\hat{\mathbf{x}}^t_0 \mid \mathbf{x}_t)$ decreases monotonically during the generative process.
\end{proof}

\end{CJK*}

\end{document}